\documentclass{article}
\usepackage[utf8]{inputenc}
\usepackage[T1]{fontenc}    
\usepackage{arxiv}

\usepackage{subfig}
\usepackage{authblk}
\title{Stochastic Gradient Descent on a Tree: an Adaptive and Robust Approach to Stochastic Convex Optimization}
\author[*]{Sattar Vakili}
\author[$\dagger$]{Sudeep Salgia}
\author[$\dagger$]{Qing Zhao}
\affil[*]{Prowler.io, Cambridge, UK, \emph{sattar@prowler.io}}
\affil[$\dagger$]{School of Electrical \& Computer Engineering, Cornell University, Ithaca, NY, \emph{\{ss3827,qz16\}@cornell.edu} }

\date{}

\usepackage[utf8]{inputenc} 
\usepackage[T1]{fontenc}    
\usepackage{hyperref}       
\usepackage{url}            
\usepackage{booktabs}       
\usepackage{amsfonts}       
\usepackage{nicefrac}       
\usepackage{microtype}      

\usepackage{amsmath}
\usepackage{amsthm}
\usepackage{bbm}
\usepackage{dsfont}
\usepackage{algorithm}
\usepackage{algpseudocode}

\usepackage{psfrag}
\usepackage{graphicx}
\usepackage{caption}
\usepackage{float,epsfig}
\usepackage{color}

\newtheorem{lemma}{Lemma}
\newtheorem{theorem}{Theorem}

\newcommand{\Log}{\lambda}
\newcommand{\1}{\mathds{1}}

\def\E{\mathbb{E}}
\def\R{\mathbb{R}}
\def\I{\mathbb{I}}
\def\N{\mathbb{N}}

\def\bp{\check{p}}

\def\nn{\nonumber}
\def\O{\mathcal{O}}
\def\X{\mathcal{X}}

\DeclareMathOperator{\sgn}{sgn}

\begin{document}

\maketitle

\begin{abstract}
Online minimization of an unknown convex function over the interval $[0,1]$
is considered under first-order
stochastic bandit feedback, which returns a random
realization of the gradient of the function at each query
point. Without knowing the distribution of the random
gradients, a learning algorithm sequentially chooses query
points with the objective of minimizing regret defined as
the expected cumulative loss of the function values at the
query points in excess to the minimum value of the
function. An approach based on devising a biased random
walk on an infinite-depth binary tree constructed through
successive partitioning of the domain of the function is
developed. Each move of the random walk is guided by a sequential test based on confidence bounds on the empirical mean constructed using  the law of the iterated logarithm. With no tuning parameters, this learning algorithm is
robust to heavy-tailed noise with infinite variance and adaptive to unknown function
characteristics (specifically, convex, strongly convex, and nonsmooth). It achieves the corresponding optimal
regret orders (up to a $\sqrt{\log T}$ or a $\log\log T$ factor) in each class of functions  and offers better or matching regret orders than the classical stochastic gradient descent approach which requires the knowledge of the function characteristics for tuning the sequence of step-sizes.    

\end{abstract}

\footnotetext{This work was supported by the National Science Foundation under Grant
CCF-1815559.}

\section{Introduction}

\subsection{Stochastic Convex Optimization}
\label{sec:intro-sco}

In stochastic convex optimization, the objective function $f(x)$ is a stochastic function given as the expectation over a random variable/vector $\xi$:
\begin{equation}
    f(x) = \E[F(x,\xi)],
    \label{eq:fx}
\end{equation}
where the design parameter $x$ is in a convex and compact set
$\mathcal{X}$. The distribution of $\xi$ may not be known, or even if it is known, the expectation over $\xi$ is difficult to evaluate analytically. As a result, the objective function $f(x)$ is unknown, except for the knowledge that it is convex. 

The above optimization problem can be cast as a sequential learning problem where the learner chooses a query point $x_t\in\mathcal{X}$ at each time $t$ and observes the corresponding random loss $F(x_t,\xi_t)$ or the random gradient $G(x_t,\xi_t)$. These two feedback models are commonly referred to, respectively, as the zeroth-order and the first-order stochastic optimization. A learning policy governs the selection of the query points $\{x_t\}_{t\ge 1}$ based on past observations, with the objective that $x_T$ converges to the minimizer $x^*=\arg\min_{x\in\mathcal{X}} f(x)$ (or $f(x_T)$ to $f(x^*)$) over a growing horizon of length $T$.

Under an online formulation of the problem, a more suitable performance measure is the cumulative regret defined as the expected cumulative loss  at the query points in excess
to the minimum loss: $R(T) =\E\left[\sum_{t=1}^T  (F(x_t,\xi_t) - f(x^*))\right]$. Under
this objective, the query process needs to balance the exploration of
the input space $\mathcal{X}$ in search for $x^*$ and the associated
loss incurred during the search process. The behavior of regret $R(T)$
over a growing horizon length $T$ is a finer measure than the
convergence of $x_T$ or $f(x_T)$. Specifically, a policy with a
sublinear regret order in $T$ implies that $f(x_T)$ converges to $f(x^*)$. The converse, however,
is not true.
In particular, the convergence of $x_T$ to $x^*$
or $f(x_T)$ to $f(x^*)$ does not imply a
sublinear, let alone an optimal, order of the regret.

An example of online stochastic convex optimization is the classification of a real-time stream of random instances $\{\xi_t\}_{t\ge 1}$ with each instance given by its feature and hidden label. Without knowing the joint distribution of the feature and label, an online learning policy chooses the classifiers $\{x_t\}_{t\ge 1}$ sequentially over time to produce online classification of the streaming instances. Empirical risk minimization using mini-batching of a large data set can also be viewed as a stochastic optimization problem~\cite{Bottou2018}, except that the resulting expectation is with respect to the random drawing of the mini-batches (often uniform with replacement) rather than the true distribution underlying the data generation.  

\subsection{Stochastic Gradient Descent}

The study of stochastic convex optimization dates back to the seminal work by  Robbins and Monro in 1951~\cite{Robbins51} under the term ``stochastic approximation.''  The problem studied there is to approximate the root of a monotone function $g(x)$ based on successive observations of random function values at chosen query points (also known as stochastic root finding~\cite{Pasupathy11}). The equivalence of this problem to the first-order stochastic convex optimization is immediate when $g(x)$ is 
the gradient of a convex loss function $f(x)$. The zeroth-order version of the problem was studied in a follow-up work by  Kiefer and Wolfowitz~\cite{Kiefer52}. 

The stochastic gradient descent (SGD) approach developed by Robbins and Monro~\cite{Robbins51} has long become a classic and is widely used. The basic idea of SGD is to choose the next query point $x_{t+1}$ in the opposite direction of the
observed  gradient while ensuring $x_{t+1}\in\mathcal{X}$ via a projection operation. Omitting the projection operation, we can write $x_{t+1}$ as
\begin{equation}
x_{t+1} = x_t - \eta_t G(x_t,\xi_t),    
\end{equation}
where $\eta_t$ is a properly chosen step-size at time $t$. Due to the noise effect of the random gradients $G(x_t; \xi_t)$, it is necessary that the step-sizes $\{\eta_t\}_{t\ge 1}$ diminishes to zero to ensure convergence of $x_t$. Since $G(x_t; \xi_t)$ contains both the signal (the true gradient $g(x_t)=\E[G(x_t; \xi_t)]$) and noise, the diminishing rate of $\{\eta_t\}_{t\ge 1}$ in $t$ needs to be carefully controlled to balance the tradeoff between learning rate and noise attenuation. Naturally, the optimal choice depends on how fast the gradient $g(x)$ approaches to zero as $x$ tends to $x^*$ and the variance of the random gradient samples.

While earlier studies on stochastic approximation focus on the convergence of $x_T$ and $f(x_T)$ (see a survey by Lai in~\cite{LaiSur}), a series of recent work has established the regret orders of SGD for different classes of functions. As shown in Tabel~\ref{Table:Comp-SGD-SGDT}, SGD offers  ${\O}(\sqrt{T}\log T)$ regret
for convex functions, ${\O}(\log^2(T))$ regret for $\alpha-$strongly convex functions, and ${\O}(\log T)$ regret for functions that are non-differentiable at $x^*$, which are near-optimal\footnote{A number of variants of SGD with various noise-reduction techniques exist in the literature that achieve the optimal regret order (see, for example,~\cite{Rakhlin12}). We consider in Table~\ref{Table:Comp-SGD-SGDT} the basic form of SGD since these noise-reduction techniques often require additional storage and computation resources and may not be suitable for online settings. An additional assumption on the smoothness of the objective function with prior knowledge on the smoothness parameter can also close the gap to the lower bounds~\cite{Shamir13}.} as compared to the lower bounds.

To achieve these near-optimal regret orders, however, it is necessary to know which category the underlying unknown objective function $f(x)$ belongs to, as well as nontrivial bounds on the corresponding parameters of the function characteristics (i.e., the parameter $\alpha$ for strong convexity and the jump in the subgradient at $x^*$ when $f(x)$ is non-differentiable at $x^*$). Such information is crucial in choosing the diminishing rate of the step-sizes $\{\eta_t\}_{t\ge 1}$, and the sensitivity of SGD to model mismatch, estimation errors in the parameters, and ill-conditioning of the functions is well documented.

\subsection{RWT: an Adaptive and Robust Approach}

We show in this work that for one-dimensional problems, an alternative approach to stochastic convex optimization self adapts to the function characteristics and offers better or matching regret orders than SGD in each class of functions without assuming any knowledge on the function characteristics.  It can also handle heavy-tailed noise with infinite variance, a case for which the applicability of SGD is unclear to our knowledge.

Referred to as Random Walk on a Tree (RWT), this policy was proposed by two of the authors of this paper in a prior work~\cite{ISITPaper} that analyzed its regret performance for convex functions under sub-Gaussian noise distributions. In this paper, we demonstrate the adaptivity of RWT to different function characteristics and robustness to heavy-tailed noise with infinite variance. We also refine the termination thresholds in the local sequence test of RWT based on the law of the iterated logarithm, which leads to improved regret orders.

The basic idea of RWT is to construct an infinite-depth binary tree
based on successive partitioning of the input space
$\mathcal{X}$. Specifically, the root of the tree corresponds to
$\mathcal{X}$, which, without loss of generality, is assumed to be
$[0,1]$ for the one-dimensional case. The tree grows to infinite depth
based on a binary splitting of each node (i.e., the corresponding
interval) that forms the two children of the node at the next level. 

The query process of RWT is based on a biased random walk on this interval tree
that initiates at the root node. Each move of the random walk is guided
by a local sequential test based on random gradient realizations drawn
from the left boundary, the middle point, and the right boundary of the
interval corresponding to the current location of the random walk. The
goal of the local sequential test is to determine, with a confidence
level greater than $1/2$, whether there is a change of sign in the
gradient in the left sub-interval or the right sub-interval of the
current node. If one is true (with the chosen confidence level),
the walk moves to the corresponding child that sees the sign change. For all other outcomes, the walk moves back to the parent of the current node.
The stopping rule and the output of the local sequential test are based
on properly constructed lower and upper confidence bounds of the
empirical mean (or truncated empirical mean in the case of infinite variance) of the observed gradient realizations. A greater than
$1/2$ bias of the random walk is sufficient to ensure convergence to the
optimal point $x^*$  at a geometric rate, regardless of the function characteristics.

By bounding the sample complexity of the local sequential test and
analyzing the trajectory of the biased random walk, we establish the regret orders of RWT as shown in Table~\ref{Table:Comp-SGD-SGDT} for sub-Gaussian distributions (a $\log\log T$ factor is omitted; see Sec.~\ref{Sec:Analysis} for the exact orders and finite-time bounds). Similar order-optimal (up to poly-$\log T$ orders) regret performance is also established for heavy-tailed distributions with infinite variance.  We are unaware of results on whether SGD can achieve sublinear regret orders under infinite noise variance.

\begin{table}[b]
\begin{center}
\begin{tabular}{ |c|c|c| c| } 
 \hline
&&&\\
       & convex & strongly convex & non-differentiable\\
       &&& at $x^*$ \\
 \hline
 &&&\\
 SGD & $\sqrt T \log T$ & $\log^2 T$ & $\log T$ \\
 &\cite{Shamir13}&\cite{Shamir13}&\cite{Lim2011}\\
 \hline
  &&&\\
 RWT & $\sqrt{T\log T}$ & $\log T$ & $\log T$ \\ 
 &&&\\
 \hline
&&&\\
 Lower Bound & $\sqrt T$ & $\log T$ & $\log T$ \\ 
 &\cite{Agrawal12}&\cite{Agrawal12}&\cite{Agrawal12}\\
 \hline
\end{tabular}
\end{center}
\caption{\label{Table:Comp-SGD-SGDT}Regret performance of SGD and RWT under sub-Gaussian noise.}
\end{table}

%

In contrast to SGD that relies on a manually controlled sequence of step-sizes to tradeoff learning rate with noise attenuation, RWT, with no tuning parameters, self adapts to function characteristics through the local sequential test that automatically draws more or fewer samples as demanded by the underlying statistical models. As shown in 
Table~\ref{Table:Comp-SGD-SGDT}, RWT outperforms or matches the regret orders of SGD without prior information on the function characteristics.

Another key difference between SGD and RWT is in the induced random walk in the input space $\mathcal{X}$. The unstructured moves of SGD  may land at any points in  $\mathcal{X}$. RWT, however, queries only a fixed set of countable number of points in $\mathcal{X}$. Furthermore, given the current location on the binary tree, the next move is restricted to only the parent and the two children of this node. This highly structured mobility allows storage-efficient caching of side observations for noise reduction at future query points.

\subsection{Other Related Work}

The classical probabilistic bisection algorithm (PBA) has been employed as a solution to stochastic root finding under a one-dimensional input space.
 Assuming a prior distribution of the optimal point $x^*$, PBA
updates the belief (i.e., the posterior distribution) of $x^*$ based on
each observation and subsequently probes the median point of the belief.
It was shown in \cite{Frazier16} that the regret order of PBA is upper bounded
by $\O(T^{0.5+\epsilon})$ for a small $\epsilon>0$, and an $\O({\sqrt T}{\log T})$ regret order was conjectured. 

There may appear to be a connection between RWT and PBA, since both
algorithms involve a certain bisection of the input domain. These two
approaches are, however, fundamentally different. First, PBA requires
the knowledge on the distribution of the random gradient function to perform the belief update, while
RWT operates under unknown models. Second, the belief-based bisection
in PBA is on the entire input domain $\mathcal{X}$ at each query and needs to be updated  based on each random observation. The interval tree in
RWT is predetermined,
and each move of the random walk leads to a bisection of a
\emph{sub-interval} of $\mathcal{X}$ that is shrinking in geometric rate
over time with high probability. It is this zooming effect of the biased
random walk that leads to a $\O(1)$ computation and memory complexity.
For PBA, if $\mathcal{X}$ is discretized to $M$ points for computation
and storage, updating and sorting the belief would incur $\O(M\log M)$
computation complexity at each query and linear memory requirement. Lastly, the regret order of RWT outperforms that of PBA.


Under the zeroth-order feedback model where the decision maker has
access to the function values, the problem can be viewed as a
continuum-armed bandit problem, on which a vast body of results exists.
In particular, the work in~\cite{Agrawal11} developed an approach based on the ellipsoid algorithm that achieves an $\O(\sqrt T (\log T)^{\frac{3}{2}})$ regret when the objective function $f$ is convex and Lipschitz. The continuum armed bandit under Lipschitz assumption (not necessarily convex) has been studied in~\cite{Agrawal95, Kleinberg05, Kleinberg08} where higher orders of regret were shown. The $\mathcal{X}$-armed bandit introduced in~\cite{Bubeck11} considered a Lipschitz function with respect to a dissimilarity function known to the learner. Under the assumption of a finite number of global optima and a particular smoothness property, an $\Tilde{\O}(\sqrt{T})$ regret was shown. While the proposed policy in~\cite{Bubeck11} uses a tree structure for updating the indexes in a bandit algorithm, it is fundamentally different from RWT in that the policy does not induce a random walk on the tree.
This line of work differs from the gradient-based approach considered in this work. Nevertheless, since an $\O(1)$ number of samples from $F$ can be translated to a sample from $G$ under certain regularity assumptions, gradient-based approaches can be extended to cases where samples from $F$ are directly fed into the learning policy. 

We mention that the stochastic online learning setting considered here is different, in problem formulation, objective, and techniques, from an adversarial counterpart of the problem where the loss function is deterministic and adversarially chosen at each time $t$. On this line of research, see~\cite{Hazan16,Shwartz12} and references therein.

\section{Problem Formulation}\label{Sec2}

We aim to minimize a stochastic convex loss function $f(x)$ as given in~\eqref{eq:fx}. Let $g(x)$ be the gradient (or sub-gradient) of $f(x)$. Let $G(x,\xi)$ be unbiased random gradient observations with $\E[G(x,\xi)]=g(x)$.

Without knowing $f(x)$ or the stochastic models of $F(x,\xi)$ or $G(x,\xi)$, a learner sequentially chooses the query points $\{x_t\}_{t=1}^\infty$, incurs i.i.d. losses $F(x_t,\xi_t)$, and  observes i.i.d. gradient samples $G(x_t,\xi_t)$. The objective is to  design a learning policy $\pi$ that is a mapping from past observations to the next query point to minimize the cumulative regret defined as 
\begin{eqnarray}\label{RegretDef}
{R}_{\pi}(T)=\mathbb{E}\left[\sum_{t=1}^T \left(F(x_{\pi(t)},\xi_t)- F(x^*,\xi_t)\right)\right],
\end{eqnarray}
where $x_{\pi(t)}$ is the query point at time $t$ under policy $\pi$.

\subsection{Function Characteristics}

The loss function $f$ is said to be convex if and only if
\begin{eqnarray}\label{convex}
f(y) \ge f(x) + g(x)(y-x),~ \forall x,y \in \X.
\end{eqnarray}
It is $\alpha$-strongly convex (for some $\alpha>0$) if and only if
\begin{eqnarray}\label{SC}
f(y) \ge f(x) + g(x)(y-x) + \frac{\alpha}{2} (y-x)^2,~ \forall x,y \in \X.
\end{eqnarray}

We also consider a nonsmooth case where $f(x)$ is non-differentiable at $x^*$. This often occurs in optimization problems that involve L1-norm regularization or have discrete parameters~\cite{Lim2011}. For such functions, there exists a lower bound $\delta>0$ on the magnitude of the (sub)-gradient: 
\begin{eqnarray}\label{nondiff}
 |g(x)| \ge \delta~~~\text{for all}~x\neq x^*.
\end{eqnarray}
In other words, the signal component in the random observations $G(x,\xi)$ does not diminish to zero as $x$ tends to $x^*$, making $\log T$ regret order possible even under noise with infinite variance. 

\subsection{Noise Characteristics}


The distribution of $G(x,\xi)-g(x)$ is said to be sub-Gaussian with parameter $\sigma^2$ if its moment generating function is bounded by that of a Gaussian random variable with variance $\sigma^2$:
\begin{eqnarray}\label{SubGAss}
\mathbb{E}\left[\exp(\lambda \left(G(x,\xi)-g(x))\right)\right] \le \exp(\frac{\lambda^2\sigma^2}{2}).
\end{eqnarray}

We also consider heavy-tailed distributions where the only assumption is the existence of a $b$-th ($b>1$) moment: 
\begin{eqnarray}\label{HTAss}
\E\left[|G(x,\xi)|^b\right]\le u,
\end{eqnarray} 
for some $u>0$. Note that this covers the class of distributions of $G(x,\xi)$ with unbounded variance.




\section{Random Walk on a Tree}\label{Sec3}

RWT is based on an infinite-depth binary tree with nodes representing a subinterval of $\mathcal{X}$. The $2^l$ nodes at depth $l$ ($l=0,1,2,\ldots$) of the tree correspond to the intervals resulting from an equal-length partition of $\mathcal{X}$, with each interval of length $2^{-l}$. Each node at depth $l$ has two children corresponding to its equal-length subintervals at depth $l+1$. Let $N_{k,l}$ ($k=1,\ldots, 2^l$, $l=0,1,\ldots$) denote the $k$th node at depth $l$. 
We use the terms node and its corresponding interval interchangeably.     

\vspace{1em}
\begin{figure}[h]
\centering
\includegraphics[width=.55\textwidth]{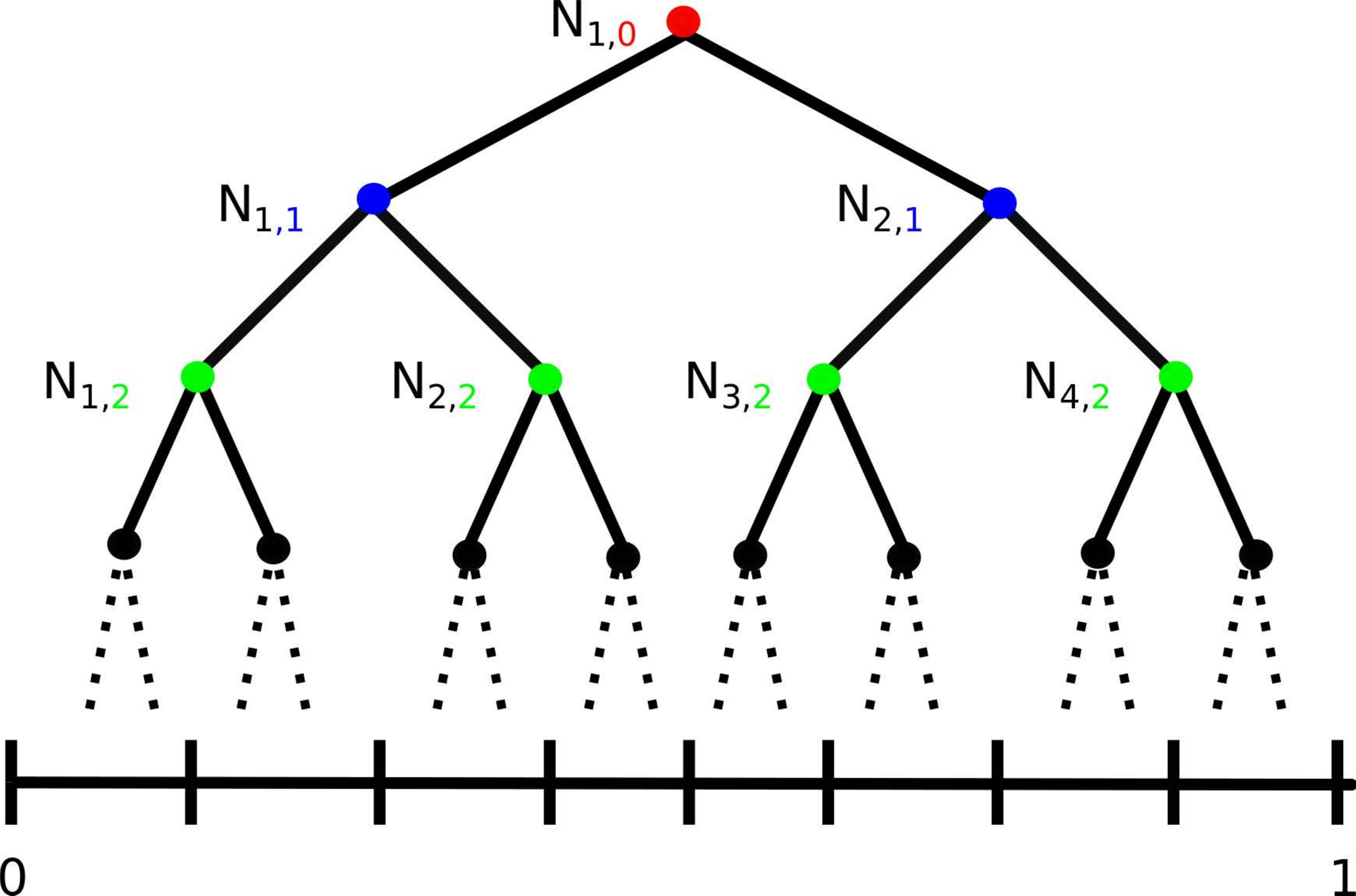}
\caption{The binary tree $\mathcal{T}$ representing the subintervals of $[0,1]$. At level $0$, $N_{1,0}$ corresponds to the interval $[0,1]$; at level $1$, $N_{1,1}$ and $N_{2,1}$, respectively, correspond to the intervals $[0,0.5]$ and $[0.5,1]$; and so on. }
\label{fig:tree}
\end{figure}

\subsection{The Biased Random Walk on the Tree}

The basic structure of RWT is to carry out a biased random walk on the interval tree. The walk starts at the root of the tree. Each move of the random walk is to one of the three adjacent nodes (i.e., the parent and the two children with the parent of the root defined as itself) of the current location. It is guided by the outputs of a confidence-bound based sequential test carried over the two boundary points and the middle point of the interval currently being visited by the random walk. 

Consider a generic sampling point $x \in[0,1]$. The goal of the sequential test is to determine, at a given confidence level, whether $g(x)$ is negative or positive. If the former is true, the test module outputs $-1$, indicating the target $x^*$ is more likely to lie on the right of the current sampling point $x$; if the latter is true, the test module outputs $1$, indicating the target $x^*$ is more likely to lie on the left of $x$ (see the next subsection on the details of the sequential test).

Based on the binary outcomes of the local sequential tests, the random walk on the tree  consists of the following loop until the end of the time horizon. 
Let $N_{k,l}$ denote the
current location of the random walk. The boundary points and the middle point of the interval corresponding to $N_{k,l}$ are probed by the sequential test module. If the output sequence on the left boundary, middle point and the
            right boundary, in order, is $\{-1,1,1\}$ (indicating a sign change in the left subinterval), the walk 
                 moves to the left child of $N_{k,l}$. If the output sequence is $\{-1,-1,1\}$, the walk
                moves to the right child of $N_{k,l}$. For all other output sequences, the walk moves back to the parent of $N_{k,l}$.

\subsection{The Local Sequential Test}

We now specify the local sequential test at a generic query point $x$. The test sequentially draws random gradient samples $G(x,\xi)$. After collecting each sample, it determines whether to terminate the test and if yes, which value to output. The termination and decision rules are chosen to satisfy a confidence level $1-\bp$ to ensure that the resulting random walk is biased toward the target $x^*$, where $\bp$ can be any value in $(0,1-\frac{1}{\sqrt[3] 2})$.  
By convention, we define the output of the test at $x=0$ to be $-1$, and at $x=1$ to be $1$, without performing the test.

The construction of the termination rule exploits the law of the iterated logarithm and depends on the noise characteristics. We consider separately the cases of  sub-Gaussian and heavy-tailed distributions. 

\subsubsection{Sub-Gaussian Distributions}

For sub-Gaussian distributed gradient samples, the test statistics can simply be the sample mean $\overline{g}_s(x)$ given by
\begin{equation}
\overline{g}_s(x) = \frac{1}{s}\sum_{t=1}^s  G(x,\xi_t).   
\end{equation}
For a given confidence level parameter $\bp$, the sequential test is given Fig.~\ref{Fig:STest-subG}, where $\sigma^2$ is the sub-Gaussian
parameter specified in~\eqref{SubGAss}.


\begin{figure}[H]
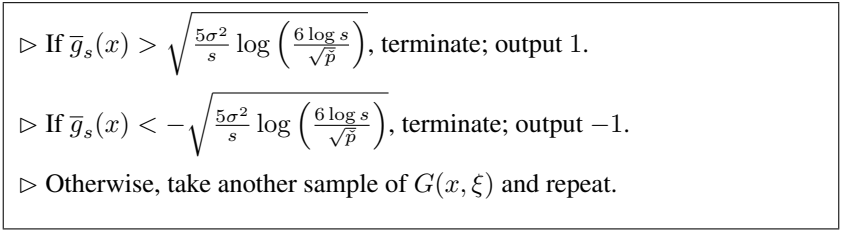

\begin{center}
\noindent\fbox{
\parbox{4.25in}{
{

$\rhd$ If $\overline{g}_s(x)>\sqrt{\frac{5 \sigma^2}{ s} \log\left(\frac{6 \log s}{\sqrt{ \bp}}\right)}$, terminate; output $1$.\\[0.5em]
$\rhd$ If $\overline{g}_s(x) <-\sqrt{\frac{5 \sigma^2}{ s} \log\left(\frac{6\log s}{\sqrt{ \bp}}\right)}$, terminate; output $-1$.\\[0.5em]
$\rhd$ Otherwise, take another sample of $G(x,\xi)$ and repeat.\\
}
}}
\caption{The sequential test at a sampling point $x$ under sub-Gaussian noise.}\label{Fig:STest-subG}
\end{center}
\end{figure}


\subsubsection{Heavy-Tailed Distributions}

For heavy-tailed distributions with a bounded $b$-th ($b>1$) moment,
define the truncated sample mean of gradient obtained from $s$ observations under a given confidence-level parameter $\bp > 0$  as follows:
\begin{eqnarray}\label{Trunc}
\widehat{g}_{s, \bp}(x) = \frac{1}{s}\sum_{t=1}^{s}G(x,\xi_t)\mathds{1}\{ |G(x,\xi_t)|\leq B_t \},
\end{eqnarray}
where 
\begin{eqnarray}\nn
B_t &=&  B_0\left( \frac{t}{\Log(t)}\right)^{\frac{1}{b}},\\\nn
\Log(t) &=& 10^b \log \left( \frac{12 \max\{ \log(t), 2 \}}{b \sqrt{ \bp}} \right),\\
B_0&=& \max  \left\{ \left( \frac{2^{\frac{2 + b}{b}}}{\Log(1)^{\frac{2-b}{b}}} \frac{15u}{3 - \sqrt{2}} \right)^{\frac{1}{b}}, \left( \frac{4 \sqrt{2}u \log 2}{ \sqrt{ \log \left(\log 3\right) } } \right)^{\frac{1}{b}}, {\left(2\sqrt{2} bu 10^{\frac{b}{2}} \right)} \right\}.
\end{eqnarray}
 In the truncated sample mean, the $t$-th sample is compared to a threshold $B_t$ and replaced with $0$ if its value exceeds the threshold. The resulting sequential test is given in Fig.~\ref{Fig:STest-HT}, where $u$ is the bound on the $b$-th moment as given in~\eqref{Trunc}.

\begin{figure}[h]
\begin{center}
\noindent\fbox{
\parbox{5.3in}{
{
\begin{itemize}
\item[$\rhd$] If $\displaystyle \widehat{g}_{s,\bp}(x)> \sqrt{\frac{B_0^2}{2}s^{\frac{2-2b}{b}} \log \left( \frac{12 \log s}{b  \sqrt{ \bp}}\right)} + \frac{1}{s} \sum_{t = 1}^{s} \frac{u}{B_t^{b-1}}$, terminate; output $1$.
\item[$\rhd$]If $\displaystyle \widehat{g}_{s,\bp }(x)<- \sqrt{\frac{B_0^2}{2}s^{\frac{2-2b}{b}} \log \left( \frac{12 \log s}{b \sqrt{ \bp}}\right)} - \frac{1}{s} \sum_{t = 1}^{s} \frac{u}{B_t^{b-1}}$, terminate; output $-1$.
\vspace{.4em}
\item[$\rhd$] Otherwise, take another sample of $G(x,\xi)$ and repeat.
\end{itemize}
}
}}
\caption{The sequential test at a sampling point $x$ under heavy-tailed noise.}\label{Fig:STest-HT}
\end{center}
\end{figure}

\section{Regret Analysis}\label{Sec:Analysis}

In this section, we provide regret analysis of RWT under variant function and noise characteristics. Corresponding to the two components---the global random walk and the local sequential test---of the policy, the analysis builds on establishing the convergence rate of the random walk towards $x^*$ and  the sample complexity of the sequential test. Each is given in a lemma in the subsequent sections.  

\subsection{The Geometric Convergence Rate of the Random Walk}

Let $n$ denote the index of the steps taken by the random walk. Let $x_{(n)}$ denote the position of the random walk after $n$ steps. In particular, $x_{(0)}$ is the root node.   
Let $\Delta_{x_{(n)}}=\max_{x\in x_{(n)}} |x-x^*|$ denote the maximum distance between a point in the interval corresponding to $x_{(n)}$ and $x^*$. Lemma~\ref{Lemma2} establishes a high-probability upper bound on $\Delta_{x_{(n)}}$ after $n$ steps are taken by the random walk.

\begin{lemma}\label{Lemma2}
With probability at least~$1-\exp(-\frac{n(2p-1)^2}{2})$, we have
\begin{eqnarray}
\Delta_{x_{(n)}}\le  2^{-\frac{n(2p-1)}{2}},
\end{eqnarray}
where $p\ge (1-\bp)^3>\frac{1}{2}$ is the bias of the walk. 
\end{lemma}

\emph{Proof:} See Appendix~A.

 Lemma~\ref{Lemma2} shows that the random walk converges at a geometric rate to $x^*$. Notice that this result is independent of the characteristics of the function or noise. 

\subsection{The Sample Complexity of the Local Sequential Test}

\subsubsection{Sub-Gaussian Distributions} The following lemma gives an upper bound on the sample complexity and  error probability of the local sequential test under sub-Gaussian distributions.

\begin{lemma}\label{SCofP}
Let $\tau(x)$ denote the termination time of the local sequential test at an arbitrary query point $x\in\X$ as given in Fig.~\ref{Fig:STest-subG}. Under sub-Gaussian distributions defined in~\eqref{SubGAss}, the sample complexity $\E[\tau(x)]$ of the local sequential test  is given by\footnotetext{In the analysis of the local sequential test, we assume that there are at least $3$ samples taken before stopping the test.}
\begin{eqnarray}
\mathbb{E}[\tau(x)]\le \frac{40 \sigma^2}{ |g(x)|^2} \log \left( \frac{12}{ \sqrt{ \bp}} \log \left( \frac{240 \sigma^2 }{\sqrt{ \bp}  |g(x)|^2} \right)\right) + 2.
\end{eqnarray}

The probabilities of an incorrect test outcome under each hypothesis on the sign of $g(x)$ are bounded as follows:
\begin{eqnarray}\nn
&\mathbb{P} \left[\overline{g}_{\tau}(x)> \sqrt{\frac{5 \sigma^2}{ \tau} \log\left(\frac{6 \log \tau}{ \sqrt{ \bp}}\right)}~|\, g(x)<0 \right]\le \bp,\\\label{ProbL2}
&\mathbb{P} \left[\overline{g}_{\tau}(x)<- \sqrt{\frac{5 \sigma^2}{ \tau} \log\left(\frac{6\log \tau}{ \sqrt{ \bp}}\right)} ~|\, g(x)>0 \right]\le \bp.
\end{eqnarray}

\end{lemma}

\emph{Proof:} See Appendix~{B}.




Lemma~\ref{SCofP} shows that the error probability of the sequential test at all query points $x\in\X$ is upper bounded by $\bp$. The condition for the random walk to move in the right direction is that the output of all three tests carried out on the boundary points and the middle point of the current interval are correct. Thus, the probability $p$ that the random walk moves in the right direction satisfies $p\ge (1-\bp)^3$ which indicates $p>\frac{1}{2}$ by the choice of $\bp \in(0,1-\frac{1}{\sqrt[3] 2})$. This ensures that the random walk is biased toward $x^*$ as required for the geometric convergence of the random walk as specified in Lemma~\ref{Lemma2}.  

To bound the test error, we employ techniques similar to the ones used in the proof of the law of iterated logarithm. By bounding the error probability for geometrically increasing intervals, the total probability of error can be bounded using the union sum and the convergence for the Riemann Zeta function for index greater than 1. The upper bound $\bp$ on the error probabilities is ensured by  choosing appropriate constants in the termination threshold.

\subsubsection{Heavy-Tailed Distributions}

Analogous to Lemma~\ref{SCofP}, we have the following result on the sample complexity and error probability of the sequential test under heavy-tailed distributions.

\begin{lemma}\label{heavy_tailed_test}
Let $\tau(x)$ denote the termination time of the local sequential test at an arbitrary query point $x\in\X$ as given in Fig.~\ref{Fig:STest-HT}. Under heavy-tailed distributions satisfying the bounded $b$-th ($b>1$) moment condition given in~\eqref{HTAss}, the sample complexity $\E[\tau(x)]$ of the local sequential test  is given by

\begin{align*}
    \E[\tau(x)] & \leq  \gamma_b \left(\left( \frac{{8}B_0^2}{|g(x)|^2} \log \left( \frac{18}{c_b} \log \left( {\frac{{144}B_0^2 }{|g(x)|^2 c_b  }} \right) \right) \right)^{\frac{b}{2(b-1)}}      + {8}\right), 
\end{align*}
where $\gamma_b = {\left(\Gamma \left(\frac{2b - 1}{b -1}\right) \left(\frac{{1}}{B_0}\left( \frac{u}{3} + \frac{1}{8} \right)\right)^{\frac{b}{b-1}} + 1\right) }$, $\displaystyle \Gamma(z) = \int_0^{\infty} x^{z-1} e^{-z} dz$ is the Gamma function and $c_{b} = (b -1) \sqrt{\bp}$.
The probabilities of an incorrect test outcome under each hypothesis on the sign of $g(x)$ 
are upper bounded by $\bp$.
\end{lemma}

\emph{Proof:} See Appendix C.

\subsection{The Cumulative Regret}

We are now ready to provide the regret performance of RWT under various cases of the function characteristics (convex, strongly convex, non-differentiable at $x^*$) and noise characteristics (sub-Gaussian, heavy-tailed).

\subsubsection{Sub-Gaussian Distributions}

The following theorem provides upper bound on regret of RWT under sub-Gaussian distributions. The regret order varies based on the function characteristics. 

\begin{theorem}\label{The:SubG}
Let $\bp\in(0,1-\frac{1}{\sqrt[3]2})$ be the chosen parameter of the sequential test and $p$ the resulting bias of the random walk. Let $g_{\footnotesize \mbox{max}}=\max_{x\in \X}g(x)$. For sub-Gaussian distributions with parameter $\sigma^2$, the regret of RWT is upper bounded as follows.
\begin{itemize}
\item For convex functions,
\end{itemize}
{\small{\begin{eqnarray}\nn
&&\hspace{-2.5em}{R}_{\mbox{\footnotesize RWT}}(T)\le \frac{6}{2p-1}\sqrt{ 10\sigma^2T\log T\log\left(\frac{12}{\sqrt{\bp}}\log\frac{2(2p-1)^2T}{3\log T \sqrt{\bp}}\right)}
+ \frac{3g_{\footnotesize \mbox{max}}}{2p-1}\sqrt{2T\log T}+ g_{\footnotesize \mbox{max}}(\log T+4).
\end{eqnarray}}}

\begin{itemize}
\item For $\alpha$-strong convexity functions, 
\end{itemize}
\begin{eqnarray}\nn
&&\hspace{-2.5em}{R}_{\mbox{\footnotesize RWT}}(T)\le
\frac{360\sigma^2\log T}{2\alpha(2p-1)^2}\log\left(\frac{12}{\sqrt{\bp}}\log\frac{2(2p-1)^2T}{3\log T\sqrt{\bp}}\right)+ \frac{18g^2_{\footnotesize \mbox{max}}\log T}{2\alpha(2p-1)^2}+ g_{\footnotesize \mbox{max}}(\log T+4). 
\end{eqnarray}

\begin{itemize}
\item For functions that are non-differentiablity at $x^*$ with a $\delta>0$ lower bound on the magnitude of gradient, 
\end{itemize}
\begin{eqnarray}\nn
&&\hspace{-2em}{R}_{\mbox{\footnotesize RWT}}(T)\le
\frac{9g_{\footnotesize \mbox{max}}\log T}{(2p-1)^2}\bigg(\frac{40\sigma^2}{\delta^2}\log\left(\frac{12}{\sqrt{\bp}}\log\frac{240\sigma^2}{\sqrt{\bp}\delta^2}\right)+2\bigg)+ g_{\footnotesize \mbox{max}}(\log T+4). 
\end{eqnarray}
\end{theorem}

\emph{Proof:} See Appendix D.

Theorem~\ref{The:SubG} shows $\O(\sqrt{T\log T\log\log T})$, $\O(\log T\log\log T)$, and $\O(\log T)$ regrets for objective functions $f(x)$ that are convex, $\alpha$-strongly convex, and non-differentiable at $x^*$, respectively. Note that while the confidence parameter $\bp$ affects the leading constants of the regret, choosing any value in $(0,1-\frac{1}{\sqrt[3]2})$ ensures these regret orders. These (near-)optimal regret orders are thus achieved without any tuning parameter or prior knowledge of the function characteristics.







\subsubsection{Heavy-Tailed Distributions} We have the following corresponding theorem for heavy-tailed distributions. 

\begin{theorem}\label{The:HT}
Let $\bp\in(0,1-\frac{1}{\sqrt[3]2})$ be the chosen parameter of the sequential test and $p$ the resulting bias of the random walk. Let $g_{\footnotesize \mbox{max}}=\max_{x\in \X}g(x)$. Under heavy-tailed distributions satisfying the bounded $b$-th ($b>1$) moment condition  in~\eqref{HTAss}, the regret of RWT is upper bounded as follows.

\begin{itemize}
\item For convex functions,
\end{itemize}
\begin{eqnarray}\nn
&&\hspace{-2.5em}{R}_{\mbox{\footnotesize RWT}}(T)\le 2\sqrt2B_0\left(
\frac{9\gamma_b}{(2p-1)^2}\right)^{\frac{b-1}{b}}T^{\frac{1}{b}}(\log T)^{\frac{b-1}{b}}\sqrt{\log \Bigg( \frac{18}{c_b}\log \left(
\frac{18}{c_b} \left(\frac{T}{9\gamma_b\log T}\right)^{\frac{2(b-1)}{b}}
\right) \Bigg)}\\\nn
&&+8\left(
\frac{9\gamma_b}{(2p-1)^2}\right)^{\frac{b-1}{b}}g_{\footnotesize \mbox{max}}T^{\frac{1}{b}}(\log T)^{\frac{b-1}{b}}+ g_{\footnotesize \mbox{max}}(\log T+4). 
\end{eqnarray}

\begin{itemize}
\item For $\alpha$-strong convexity functions, 
\end{itemize}
\begin{eqnarray}\nn
&&\hspace{-2.5em}{R}_{\mbox{\footnotesize RWT}}(T)\le \frac{4B_0^2}{\alpha}\left(
\frac{9\gamma_b}{(2p-1)^2}\right)^{\frac{2(b-1)}{b}}T^{\frac{2-b}{b}}(\log T)^{\frac{2(b-1)}{b}}\log \Bigg( \frac{18}{c_b}\log \left(
\frac{18}{c_b} \left(\frac{T}{9\gamma_b\log T}\right)^{\frac{2(b-1)}{b}}
\right) \Bigg)\\\nn
&&+\frac{4}{\alpha}\left(
\frac{9\gamma_b}{(2p-1)^2}\right)^{\frac{2(b-1)}{b}}g^2_{\footnotesize \mbox{max}}T^{\frac{2-b}{b}}(\log T)^{\frac{2(b-1)}{b}}+ g_{\footnotesize \mbox{max}}(\log T+4). 
\end{eqnarray}

\begin{itemize}
\item For functions that are non-differentiablity at $x^*$ with a $\delta>0$ lower bound on the magnitude of gradient,
\end{itemize}
\begin{eqnarray}\nn
&&\hspace{-2.5em}{R}_{\mbox{\footnotesize RWT}}(T)\le \frac{9
g_{\footnotesize \mbox{max}}\gamma_b\log T}{(2p-1)^2} \bigg( \frac{2B_0^2}{\delta^2} \log \left( \frac{18}{c_{b}} \log \left( {\frac{36B_0^2 }{\delta^2 c_{b}  }} \right) \right)^{\frac{b}{2(b-1)}} +8\bigg) + g_{\footnotesize \mbox{max}}(\log T+4). 
\end{eqnarray}

\end{theorem}

\emph{Proof:} See Appendix E.

Theorem~\ref{The:HT} shows $\O(T^{\frac{1}{b}}\log T^{\frac{b-1}{b}}({\log\log T})^{\frac{1}{2}})$, $\O(T^{\frac{2-b}{b}}\log T^{\frac{2(b-1)}{b}}\log\log T))$, and $\O(\log T)$ regrets for functions that are convex, $\alpha$-strongly convex, and non-differentiable at $x^*$, respectively. They match the corresponding lower bounds~\cite{Nemirovski1983} (up to poly-$\log T$ factors in the first two cases). 



\section{Simulation}\label{Simul}

In this section, we present simulation examples to demonstrate the adaptivity of RWT and its performance as compared to SGD. We also illustrate the use of local caching of side observations by exploiting the highly structured mobility of RWT in the input space.  

\subsection{Adaptivity of RWT and Performance Comparison with SGD}
We consider the following objective function over $\mathcal{X} = [0,1]$:
\[
f(x)=a|x - x^*|^{b},
\]
where $a > 0$, $b\geq 1$. This function is strongly convex with a strong-convexity parameter $\alpha = ab(b-1)(\max\{x^*, 1 - x^*\})^{b-2}$. The gradient is given by
\[
g(x) = ab \sgn(x-x^*)|x - x^*|^{b-1}. 
\]
The stochastic component in the gradient is modelled by additive Gaussian noise of zero mean and unit variance. Specifically, $G(x, \xi) = g(x) + \xi$ where $\xi \sim \mathcal{N}(0,1)$. 

RWT is carried out as described in Section \ref{Sec3} with parameter $\bp = 0.2$. For  SGD, we employ the standard implementation which involves generating a sequence of points $\{x_t\}_{t = 1}^{\infty}$ according to the update rule $x_{t + 1} = \mathrm{proj}_{\mathcal{X}}(x_t - \eta_t G(x_t, \xi_t))$. Here, $\eta_t$ is the sequence of step sizes that are generally chosen depending on the knowledge about the function and its corresponding parameters. The initial point $x_1$ is chosen uniformly at random in $\mathcal{X}$.


\begin{figure}
    \centering
    \includegraphics[scale= 0.4]{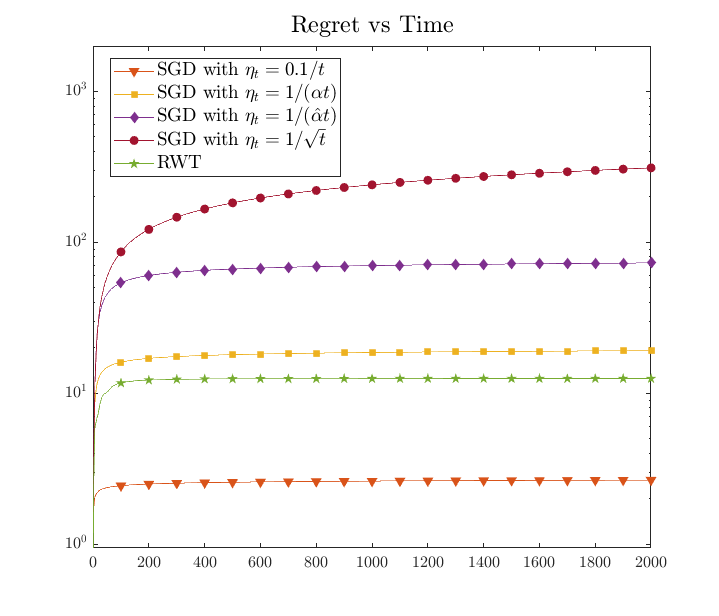}
    \caption{Performance comparison of RWT and SGD ($f(x)= 4|x - 0.2|^{1.2}$; results based on $1000$ Monte Carlo runs).}
    \label{Fig:sensitivity_SGD}
\end{figure}

In Figure \ref{Fig:sensitivity_SGD}, we compare the regret performance of RWT with that of SGD using four sets of step sizes chosen based on different levels of knowledge about the objective function $f(x)$: (i) the numerically optimized step size $\eta_t = 0.1/t$ obtained through numerical search; (ii) the commonly adopted order-optimal step size $\eta_t = 1/(\alpha t)$ chosen  with the knowledge of $f(x)$ being strongly convex and the exact value of the strong-convexity parameter $\alpha$; (iii) the order-optimal step size $\eta_t = 1/ (\hat{\alpha} t)$ chosen based on a lower bound $\hat{\alpha} = \alpha/4$ of the strong-convexity parameter; (iv) the commonly adopted order-optimal step size $\eta_t = 1/\sqrt{t}$ for general convex functions (i.e., without the knowledge of $f(x)$ being strong convex). Figure \ref{Fig:sensitivity_SGD} clearly demonstrates the sensitivity of SGD to the choice of step sizes. Imperfect knowledge on the function characteristics and/or the associated parameter results in orders of magnitude performance degradation. Except with the numerically optimized step size, which is infeasible in practice, SGD yielded inferior performance to RWT which was run without any parameter tuning. 

\begin{figure}
    \centering
    \includegraphics[scale = 0.45]{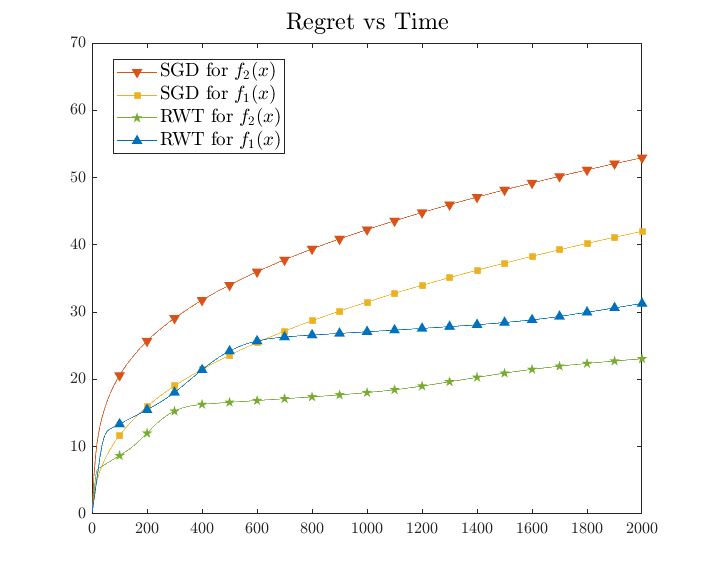}
    \caption{Adaptivity of RWT (results based on $1000$ Monte Carlo runs).}
    \label{Fig:adaptivity}
\end{figure}

In Figure~\ref{Fig:adaptivity}, we demonstrate the adaptivity of RWT that automatically takes advantage of well-behaving objective functions. In this setup, we first consider an objective function $f_1(x)$ that is not strongly convex:
\[
f_1(x)=  3|x - 0.2|^{1.6} - 1.5744|x - 0.2|^2.
\]
SGD was run with an order-optimal step size $\eta_t = 1/\sqrt{t}$ for convex functions. RWT was run with $\bp = 0.2$ as in the previous example. We then consider a strongly convex function $f_2(x)=3|x - 0.2|^{1.6}$ that retains only the first term of $f_1(x)$. Both SGD and RWT were run without the knowledge about this change in the objective function, hence with the same choices of step size and $\bp$. Figure~\ref{Fig:adaptivity} shows that RWT adapted to the unknown change in the function characteristic and offered improved performance under the strongly convex $f_2(x)$. In contrast, the performance of SGD degraded under a better behaving objective function due to the mismatch of the step size.




\subsection{Local Data Caching}

In certain applications, it may be possible to observe the gradient $G(x, \xi)$ at multiple input values of $x$ in addition to the chosen action $x_t$ which determines the regret. Such a feedback model is often referred to as the semi-bandit feedback. It is natural that these additional observations can speed up the learning process and reduce the cumulative regret. Obtaining and storing such side observations, however, come with costs in terms of computation and storage. Consider, for example, the application of online classification of a real-time stream of random instances as discussed in Sec.~\ref{sec:intro-sco}. While the random loss/gradient can be computed for all possible classifiers $x\in\mathcal{X}$ for each instance $\xi_t$, computation and storage constraints may limit such side observations to a few strategically chosen input values in $\mathcal{X}$ that best assist future learning. 

Due to the highly structured mobility of RWT in the input space $\mathcal{X}$, the input values that will be chosen in the few future steps are known to be from a small set, that is, the few neighbors of the current node on the binary tree. As a result, RWT allows storage-efficient caching of side observations from future query points. 

We consider here an intuitive local data caching scheme based on a priority queue. Specifically, after each move of the random walk, a priority index is assigned to the neighbors of the current location of the walk, where the index is determined in inverse proportion to the distance (in terms of the number of hops) on the tree. Observations are then drawn from the query points associated with the nodes in the priority queue, starting from the head of the queue up to what is allowed by the cache size. Specifically, given a cache size $c$, $c$ sequential tests are run in parallel at each time with input values chosen based on the order given by the priority queue. Note that the three input values associated with the current location of the walk are at the head of the queue (they are of distance $0$), and the shared input values across neighboring nodes are not repeated in the queue. Once a local test terminates, the binary output is noted, and the next input value in the queue is queried. Once the random walk moves to a new location, the priority queue is updated, and the process repeats. 

\begin{figure}
    \centering
    \includegraphics[scale = 0.45]{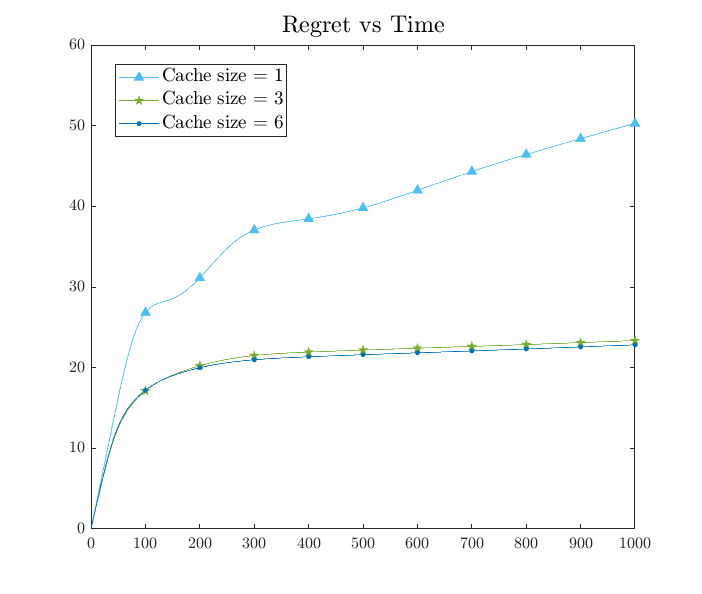}
    \caption{Local data caching to speed up learning ($f(x) = |x - 0.05|^{1.4}$;results based on $1000$ Monte Carlo runs).}
    \label{Fig:Caching}
\end{figure}

Shown in Figure \ref{Fig:Caching} is the regret performance of RWT implemented with cache sizes of $1$, $3$, and $6$. A cache size of $1$ gives the original implementation. We observe that with a small cache size of $3$, over $50\%$ of reduction in regret can be achieved. The gain with an increased cache size, however, is minimal. A more refined usage of cache may be necessary to fully exploit the increased storage capacity.

\section{Conclusion}

We gave a relatively complete regret analysis of the Random-Walk-on-a-Tree (RWT) policy for  stochastic convex optimization under various function and noise characteristics. Comparing with the popular SGD approach which requires careful tuning of the step-sizes based on prior knowledge of the function characteristics, RWT, with no tuning parameters, self adapts to unknown function characteristics and offers better or matching regret orders as SGD. The adaptivity is achieved via a local sequential test with termination thresholds designed based on the law of the iterated logarithm. The highly structured random walk also enables storage-efficient local data caching for noise reduction at future query points. We further established (near-)optimal regret orders for RWT under heavy-tailed noise with unbounded variance.Our ongoing work on extending RWT to high-dimensional problems by integrating it with coordinate minimization has shown promising results.


\newpage

\newpage
\onecolumn

\section*{Appendix A}

\begin{proof}[Proof of Lemma~\ref{Lemma2}]
We define the value $W_n$ of each step $n$ of the random walk as $W_n = 1$ if the random walk moves in the right direction, i.e., the random walk moves to the child who contains $x^*$ or moves to the parent if neither of the children contain $x^*$; and $W_n = -1$ otherwise. We also use $V_n=\sum_{m=1}^n W_m$ to denote the cumulative value of the steps. 

The condition for $W_n=1$ is that the result of all three local sequential tests at step $n$ are true. Thus, as a result of Lemma~\ref{SCofP}, we have $W_n=1$ with probability $p\ge(1-\alpha)^3>\frac{1}{2}$ which indicates $\E[W_n]\ge 2p-1>0$. The positive expected value of each step in the random walk indicates the random walk is more likely to move closer to $x^*$ rather than away from it. In particular, 
\begin{eqnarray}\nn
&&\hspace{-2em}\Pr\left[V_n\le \frac{n(2p-1)}{2}\right]\\\nn &=& \Pr\left[V_n-n(2p-1) \le -\frac{n(2p-1)}{2}\right]\\\nn
&\le& \Pr\left[V_n-\E[V_n]\le -\frac{n(2p-1)}{2}\right]\\\nn
&\le& \Pr\left[\frac{V_n}{n}-\frac{\E[V_n]}{n}\le -\frac{(2p-1)}{2}\right]\\\nn
&\le& \exp\left(-\frac{n(2p-1)^2}{2}\right). 
\end{eqnarray}

The last inequality is based on Hoeffding inequality on independent Bernoulli random variables $W_n$. 
Each time the random walk moves one step in the right direction the value of $\Delta_{x_{(n)}}$ in the local sequential tests is divided by half. For example when $V_n=0$ we have, trivially, $\Delta_{x_{(m)}}\le1$. If the random walk moves to the child who contains $x^*$ then we have $\Delta_{x_{(n)}}\le0.5$, and so on. Thus, for
$V_n> \frac{n(2p-1)}{2}$, we have
\begin{eqnarray}\nn
\Delta_{(n)}&\le& 2^{-V_n}\\\nn
&\le& 2^{-\frac{n(2p-1)}{2}},
\end{eqnarray}
which completes the proof. 

\end{proof}

\section*{Appendix B}\label{AppendixB}

\begin{proof}[Proof of Lemma~\ref{SCofP}]
The samples obtained at any query point are sub-Gaussian random variables with mean $g(x)$. We make the notation concise by using $\mu = g(x)$.
Let $\{X_i\}_{i\in \mathbb{N}}$ be a set of iid sub-Gaussian random variables with mean $\mu$ with sub-Gaussian parameter given by $\sigma^2$. Define $\displaystyle Z_n = \sum_{i = 1}^n X_i$ and $Y_n = Z_n - n\mu$. The threshold is given as $\displaystyle \tau'(s, \beta) = \sqrt{\frac{5 \sigma^2}{ s} \log \left( \frac{\log s}{\beta}\right)}$ where $\beta = \dfrac{\sqrt{\bp}}{6}$. The main idea of the proof follows arguments similar to the proof of Law of Iterated Logarithm.

Note that the sequence $Y_n$ forms a martingale and hence for any $t \in \R$, $e^{tY_n}$ forms a submartingale. Using the Doob's inequality for submartingales, we can write 
$$ P\left( \max_{1 \leq j \leq n} e^{tY_j} \geq e^{t \nu} \right) \leq \E \left[ e^{tY_n} \right] e^{-t \nu} \leq e^{\frac{n t^2 \sigma^2}{2} - t \nu}\leq  e^{ - \frac{ \nu^2}{2 n\sigma^2}}.$$
The second step uses the sub-Gaussianity of $Y_n$ and the last step uses $t = \frac{ \nu}{2 n \sigma^2}$. Hence, we can conclude that $\displaystyle P\left( \max_{1 \leq j \leq n} Y_j \geq \nu \right) \leq  e^{ - \frac{ \nu^2}{2n \sigma^2}} $.

We assume $\mu < 0$. The proof for the case $\mu > 0 $ is similar. Let $\theta = 1.25$ and $\mathbb{P}(\text{err})$ denote the probability of error in the local test. Hence, we have,
\begin{align*}
\mathbb{P}(\text{err}) & =  \Pr \left( \bigcup_{k} \left( \left\{\frac{Z_{k'}}{k'} > - \tau'(k', \beta) \  \forall \ k' < k \right\} \cap \left\{\frac{Z_{k}}{k} > \tau'(k, \beta) \right\} \right) \right) \\ 
& \leq  \Pr \left( \bigcup_{k} \left( \frac{Z_{k}}{k} > \tau'(k, \beta)  \right)  \right) \\ 
& \leq  \Pr \left( \bigcup_{k} \left( \frac{Z_{k}}{k} - \mu > \tau'(k, \beta)  - \mu \right) \right) \\ 
& \leq  \Pr \left( \bigcup_{k} \left( \frac{Y_k}{k} > \sqrt{\frac{5 \sigma^2}{k} \log \left( \frac{\log k}{\beta}\right)} \right) \right) \\ 
& \leq  \Pr \left( \bigcup_{m \in \N} \bigcup_{k = \theta^{m-1}}^{\theta^m} \left( Y_k > k \sqrt{\frac{5 \sigma^2}{k} \log \left( \frac{\log k}{\beta}\right)} \right) \right) \\ 
& \leq \sum_{m = 1}^{\infty} \Pr \left( \bigcup_{k = \theta^{m-1}}^{\theta^m} \left( Y_k > k \sqrt{\frac{5 \sigma^2}{k} \log \left( \frac{\log k}{\beta}\right)} \right) \right).
\end{align*}
Note that the above expression on the RHS is well defined only for $k \geq 3$. We consider the termination only after at least $3$ samples are taken. With $\displaystyle m_0 = \left\lfloor \frac{\log 3}{\log 1.25} \right\rfloor + 1 = 5$
\begin{align*}
\mathbb{P}(\text{err}) & \leq  \sum_{m = m_0}^{\infty} \Pr \left( \bigcup_{k = \theta^{m-1}}^{\theta^m} \left( Y_k > \sqrt{5 \sigma^2 k \log \left( \frac{\log k}{\beta}\right)} \right) \right) \\ 
& \leq \sum_{m = m_0}^{\infty} \Pr \left( \bigcup_{k = \theta^{m-1}}^{\theta^m} \left( Y_k > \sqrt{{5 \sigma^2 }{\theta^{m-1}} \log \left( \frac{\log \theta^{m-1}}{\beta}\right)} \right) \right) \\ 
& \leq  \sum_{m = m_0}^{\infty} \Pr \left( \max_{k \leq \theta^m} Y_k > \sqrt{{5 \sigma^2}{\theta^{m-1}} \log \left( \frac{\log \theta^{m-1}}{\beta}\right)}  \right) \\
& \leq  \sum_{m = m_0}^{\infty} \exp \left( - \frac{\theta^{m-1}}{2 \sigma^2} \frac{5 \sigma^2}{\theta^{m}} \log \left( \frac{\log \theta^{m-1}}{\beta}\right) \right) \\
& \leq  \sum_{m = m_0}^{\infty} \exp \left( -   \frac{2.5}{\theta} \log \left( \frac{(m-1)\log (\theta)}{\beta} \right) \right) \\
& \leq  \sum_{m = m_0}^{\infty} \left( \frac{\beta}{(m-1) \log(1.25)} \right)^{2} \\
& \leq  \left(\frac{\sqrt{\bp}}{6\log(1.25)}\right)^2 \sum_{m = 1}^{\infty} \frac{1}{m^2} \\
& \leq  \bp \left(\frac{1}{6\log(1.25)}\right)^2 \frac{\pi^2}{6} \\
& \leq  \bp 
\end{align*}
where the second step uses that $\displaystyle k \log \left( \frac{\log k}{\beta}\right) $ is increasing in $k$.

We now prove the upper bound on the number of samples required in the local test. Define $\displaystyle s_0 =  \frac{20 \sigma^2}{\mu^2} \log \left( \frac{2}{\beta} \log \left( \frac{40 \sigma^2}{\beta \mu^2} \right)\right) $ and $s_1 = \lceil s_0 \rceil$ where $\displaystyle \beta = \frac{\sqrt{ \bp}}{6}$.  We have,
\begin{align}
    \tau'(s_1, \beta) & = \sqrt{\frac{5 \sigma^2}{s_1} \log \left( \frac{\log s_1}{\beta}\right)}\nonumber \\
    & \leq \sqrt{\frac{5 \sigma^2}{s_0} \log \left( \frac{\log s_0}{\beta}\right)} \nonumber\\
    & \leq \sqrt{\frac{5 \sigma^2}{s_0} \log \left( \frac{1}{\beta} \log \left( \frac{20 \sigma^2}{\mu^2} \log \left( \frac{2}{\beta} \log \left( \frac{40 \sigma^2}{\beta \mu^2} \right)\right) \right) \right)} \nonumber\\
    & \leq \sqrt{\frac{5 \sigma^2}{s_0} \log \left( \frac{1}{\beta} \log \left( \frac{20 \sigma^2}{\mu^2} \right) + \frac{1}{\beta} \log \left( \log \left( \frac{2}{\beta} \log \left( \frac{40 \sigma^2}{\beta \mu^2} \right)\right) \right) \right)} \nonumber\\
    & \leq \sqrt{\frac{5 \sigma^2}{s_0} \log \left( \frac{1}{\beta} \log \left( \frac{20 \sigma^2}{\mu^2} \right) + \frac{1}{\beta} \log \left( \frac{80 \sigma^2}{\beta^2 \mu^2}\right) \right)} \nonumber\\
    & \leq \sqrt{\frac{\mu^2}{4} \frac{\log \left( \frac{1}{\beta} \log \left( \frac{1600 \sigma^4}{\beta^2 \mu^4} \right) \right)}{\log \left( \frac{2}{\beta} \log \left( \frac{40 \sigma^2}{\beta \mu^2} \right)\right)}} \nonumber\\
    &\leq \frac{\mu}{2} \label{mu2om}
\end{align}
where the second step follows from the fact that $\tau'(s, \beta) $ is decreasing for all $s \geq 2$.
Let $\tau$ denote the random number of samples taken before the local test terminates. Then, for $n\ge s_1$,
\begin{eqnarray}\nn
&&\hspace{-3em}\Pr[\tau> n]\\\nn
&\le&\Pr\bigg[\forall s\le n:\overline{Z}_{s}+\sqrt{\frac{5 \sigma^2}{s} \log \left( \frac{6 \log s}{\sqrt{ \bp}}\right)}>0,~\text{and}~
\overline{Z}_{s}-\sqrt{\frac{5 \sigma^2}{s} \log \left( \frac{6 \log s}{\sqrt{6 \bp}}\right)}<0\bigg ]\\\nn 
&\le&\Pr\bigg[\forall s \le n:\overline{Z}_{s}-\sqrt{\frac{5 \sigma^2}{s} \log \left( \frac{6 \log s}{\sqrt{\bp}}\right)}<0\bigg ]\\
&\le&\Pr\bigg[\forall s \le n:\overline{Z}_{s}-\sqrt{\frac{5 \sigma^2}{s} \log \left( \frac{\log s}{\beta}\right)}<0\bigg ]\\ \label{XSS}
&\le&\Pr\bigg[\overline{Z}_{n}-\mu<-\frac{\mu}{2}\bigg]\\ \label{XSS1}
&\le&\exp(- n \mu^2/(8 \sigma^2) ).
\end{eqnarray}
The inequality~\eqref{XSS} holds based on~\eqref{mu2om}, for $n\ge s_1$, and the inequality~\eqref{XSS1} holds based on the Chernoff-Hoeffding bound. We can write $\mathbb{E}[\tau]$ in terms of the sum of the tail probabilities $\Pr[\tau\ge n]$ as
\begin{eqnarray}\nn
\mathbb{E}[\tau]&=&\sum_{n= 0}^\infty\Pr[\tau\ge n]\\\nn
&=&s_1+\sum_{n= s_1+1}^\infty\Pr[\tau\ge n]\\\nn
&\le&s_1+\sum_{n= s_1+1}^\infty\exp(- n \mu^2/(8 \sigma^2) )\\\nn
&\le&s_1 +\frac{8 \sigma^2}{\mu^2}\exp(- s_1\mu^2/(8 \sigma^2))\\\nn
&\le&s_1 +\frac{8 \sigma^2}{\mu^2}\\\nn
&\le&2s_1,
\end{eqnarray}
which gives that $\displaystyle \mathbb{E}[\tau] \leq \frac{40 \sigma^2}{\mu^2} \log \left( \frac{12}{ \sqrt{ \bp}} \log \left( \frac{240 \sigma^2}{\sqrt{ \bp} \mu^2} \right)\right) + 2 $. Since similar analysis can be carried out for $\mu < 0$, we can conclude that for any point $x \in \mathcal{X}$, we have $\displaystyle \mathbb{E}[\tau(x)] \leq \frac{40 \sigma^2}{|g(x)|^2} \log \left( \frac{12}{ \sqrt{ \bp}} \log \left( \frac{240 \sigma^2}{\sqrt{ \bp} |g(x)|^2} \right)\right) + 2 $.

\end{proof}

\section*{Appendix C}
\begin{proof}[Proof of Lemma~\ref{heavy_tailed_test}]

The proof is similar to that of Lemma~\ref{Lemma2}. Let $\mu = g(x)$.  
Let $\{X_i\}_{i \in \N}$ be a set of i.i.d. random variables with mean $\mu$ and $b^{\text{th}}$ raw moment ($b \in (1, 2)$) bounded by some $u > 0$. Let $S_n$ denote the truncated mean given as $\displaystyle S_n = \sum_{t = 1}^n X_t \1_{\{|X_t| \leq B_t\}}$ where $\displaystyle B_t =  B_0\left( \frac{t}{\Log(t)}\right)^{\frac{1}{b}}$ and $\displaystyle \Log(x) = 10^b \log \left( \frac{ \max\{ \log(x), 2 \}}{\beta} \right)$ with $\displaystyle \beta = \frac{b\sqrt{\bp}}{12}$.  The constant $B_0$ is given as $\displaystyle \max \left\{ \left( \frac{2^{\frac{2 + b}{b}}}{\Log(1)^{\frac{2-b}{b}}} \frac{15u}{3 - \sqrt{2}} \right)^{\frac{1}{b}}, \left( \frac{4 \sqrt{2}u \log 2}{ \sqrt{ \log \left(\log 3\right) } } \right)^{\frac{1}{b}}, {\left(2\sqrt{2} bu 10^{\frac{b}{2}} \right)} \right\}$. The threshold used in the local test is given as $\displaystyle \tau'(s, \beta) = \sqrt{\frac{B_0^2}{2}s^{\frac{2-2b}{b}} \log \left( \frac{ \log s}{\beta}\right)} {+} \frac{1}{s} \sum_{t = 1}^{s} \frac{u}{B_t^{b-1}}$.
%
Now consider,
\begin{align*}
	S_n - n \E[X] & = \sum_{t = 1}^n X_t \1_{\{|X_t| \leq B_t\}} - n \E[X] \\
	& = \sum_{t = 1}^n X_t \1_{\{|X_t| \leq B_t\}} - \sum_{t = 1}^n \E [X_t \1_{\{|X_t| \leq B_t\}}] + \sum_{t = 1}^n \E[X_t \1_{\{|X_t| \leq B_t\}}] - n \E[X] \\
	& = \sum_{t = 1}^n \left( X_t \1_{\{|X_t| \leq B_t\}} -  \E [X_t \1_{\{|X_t| \leq B_t\}}] \right) + \sum_{t = 1}^n \left( \E[X_t \1_{\{|X_t| \leq B_t\}}] -  \E[X] \right) \\
	& = \sum_{t = 1}^n \left( X_t \1_{\{|X_t| \leq B_t\}} -  \E [X_t \1_{\{|X_t| \leq B_t\}}] \right) - \sum_{t = 1}^n \E[X \1_{\{|X| > B_t\}}]  
\end{align*}
Let $\{\phi_t\}_{t=1}^\infty$ be the sequence of random variables given by $X_t \1_{\{|X_t| \leq B_t\}} -  \E [X_t \1_{\{|X_t| \leq B_t\}}]$. The random variable $\phi_t$ has zero mean finite variance for any $t$ (since it is bounded). Furthermore, $\displaystyle \E[\phi_t^2] = \E[ (X_t \1_{\{|X_t| \leq B_t\}} -  \E [X_t \1_{\{|X_t| \leq B_t\}}])^2] \leq \E[X_t^2 \1_{\{|X_t| \leq B_t\}}] \leq u B_t^{2-b}$ and $|\phi_t| \leq B_t + |\E [X_t \1_{\{|X_t| \leq B_t\}}]| \leq 2B_t$.

Thus, the sum $\displaystyle \psi_n = \sum_{t = 1}^n \phi_t$ forms a martingale where the difference sequence is bounded and the increments are independent. The Freedman Inequality implies that for any $\nu > 0 $
\begin{align}
	P\left( \bigcup_{j = 1}^n \left( \psi_j \geq \nu \right) \right) &  \leq \exp \left(-\frac{\nu^2}{2(2B_n v/3 + \sum_{t = 1}^n\E[\phi_t^2] )} \right) \nonumber\\
	&  \leq \exp \left(-\frac{\nu^2}{2(2B_n v/3 + \sum_{t = 1}^n u B_t^{2 - b} )} \right) \nonumber\\\label{Freedman_bound}
	&  \leq \exp \left(-\frac{\nu^2}{2(2B_n v/3 +  u\left(\frac{B_0^b}{\log (2/\beta)}\right)^{\frac{2 -b}{b}} (n + 1)^{2/b} )} \right) 
\end{align}

where the last step follows from the relation that $\displaystyle \sum_{t = 1}^n \left(\frac{t}{\Log(t)}\right)^{\frac{2 - b}{b}} \leq \sum_{t = 1}^n \left(\frac{t}{\Log(1)}\right)^{\frac{2 - b}{b}} \leq 
\left(\frac{1}{\Log(1)}\right)^{\frac{2 - b}{b}} (n + 1)^{2/b}$. 

Also note that $\displaystyle \frac{1}{s} \sum_{t = 1}^s \frac{u}{B_t^{b-1}} \geq {\bigg|}\frac{1}{s} \sum_{t = 1}^s \E[X \1_{\{|X| > B_t\}}] {\bigg|} $. This follows from the result obtained in \cite{Bubeck13}.

We assume $\mu < 0$. The proof is similar for the case $\mu>0$. Let $\theta = (1.25)^{\frac{b}{2}}$ and $\mathbb{P}(\text{err})$ denote the probability of error in the local test. We have 
\begin{align*}
\mathbb{P}(\text{err}) & =  \Pr \left( \bigcup_{k} \left( \left\{\frac{S_{k'}}{k'} > - \tau'(k' , \beta) \  \forall \ k' < k \right\} \cap \left\{\frac{S_{k}}{k} > \tau'(k', \beta) \right\} \right) \right) \\ 
& \leq  \Pr \left( \bigcup_{k} \left( \frac{S_{k}}{k} > \tau'(k, \beta)  \right)  \right) \\ 
& \leq  \Pr \left( \bigcup_{k} \left( \frac{S_{k}}{k} - \mu > \tau'(k, \beta)  - \mu \right) \right) \\ 
& \leq  \Pr \left( \bigcup_{k} \left( \frac{\psi_k}{k} - \frac{1}{k}\sum_{t = 1}^k \E[X \1_{\{|X| > B_t\}}] > \tau'(k, \beta)  - \mu \right) \right) \\
& \leq  \Pr \left( \bigcup_{k} \left( \frac{\psi_k}{k} > \tau'(k, \epsilon, \beta)  + \frac{1}{k}\sum_{t = 1}^k \E[X \1_{\{|X| > B_t\}}] \right) \right) \\
& \leq  \Pr \left( \bigcup_{k} \left( \frac{\psi_k}{k} > \sqrt{\frac{B_0^2}{2}k^{\frac{2-2b}{b}} \log \left( \frac{\log k}{\beta}\right)}+ \frac{1}{s} \sum_{t = 1}^{s} \frac{u}{B_t^{b-1}}  + \frac{1}{k}\sum_{t = 1}^k \E[X \1_{\{|X| > B_t\}}] \right) \right) \\
& \leq  \Pr \left( \bigcup_{k} \left( \frac{\psi_k}{k} > \sqrt{\frac{B_0^2}{2}k^{\frac{2-2b}{b}} \log \left( \frac{\log k}{\beta}\right)} \right) \right) \\ 
& \leq  \Pr \left( \bigcup_{m \in \N} \bigcup_{k = \theta^{m-1}}^{\theta^m} \left( \psi_k > k \sqrt{\frac{B_0^2}{2}k^{\frac{2-2b}{b}} \log \left( \frac{\log k}{\beta}\right)} \right) \right) \\ 
& \leq \sum_{m = 1}^{\infty} P \left( \bigcup_{k = \theta^{m-1}}^{\theta^m} \left( \psi_k > k \sqrt{\frac{B_0^2}{2}k^{\frac{2-2b}{b}} \log \left( \frac{\log k}{\beta}\right)} \right) \right) 
\end{align*}
As before, we consider $k \geq 3$and with $\displaystyle m_0 = \left\lfloor \frac{2\log 3}{b\log (1.25)} \right\rfloor + 1 \geq 5$
\begin{align*}
\Pr(\text{err}) & \leq  \sum_{m = m_0}^{\infty} \Pr \left( \bigcup_{k = \theta^{m-1}}^{\theta^m} \left( \psi_k > \sqrt{ \frac{B_0^2}{2} k^{\frac{2}{b}} \log \left( \frac{\log k}{\beta}\right)} \right) \right) \\ 
& \leq \sum_{m = m_0}^{\infty} \Pr \left( \bigcup_{k = \theta^{m-1}}^{\theta^m} \left( \psi_k > \sqrt{\frac{B_0^2}{2} \theta^{2(m-1)/b} \log \left( \frac{\log \theta^{m-1}}{\beta}\right)} \right) \right) \\ 
& \leq \sum_{m = m_0}^{\infty} \Pr \left( \bigcup_{k = 1}^{\theta^m} \left( \psi_k > \sqrt{\frac{B_0^2}{2} \theta^{2(m-1)/b} \log \left( \frac{\log \theta^{m-1}}{\beta}\right)} \right) \right) 
\end{align*}
where the second step follows from $k^{2/b} \log \left( \frac{\log k}{\beta}\right) $ being increasing in $k$.
For the RHS, we can use the bound obtained in (\ref{Freedman_bound}). Let $\lambda(m , \theta)$ denote the value of bound in (\ref{Freedman_bound}) with $\displaystyle \nu = \sqrt{\frac{B_0^2}{2} \theta^{2(m-1)/b} \log \left( \frac{\log \theta^{m-1}}{\beta}\right)}$ and $N = \theta^m$. Thus, we have,
\begin{align*}
	\lambda(m, \theta) & = \exp \left( - \left( \frac{B_0^2\theta^{\frac{2(m-1)}{b}}}{4} \log \left( \frac{\log \theta^{m-1}}{\beta}\right) \right)\left(\frac{2B_N}{3} \sqrt{\frac{B_0^2\theta^{\frac{2(m-1)}{b}}}{2} \log \left( \frac{\log \theta^{m-1}}{\beta}\right)} +  u\left(\frac{B_0^b}{\Log(1)}\right)^{\frac{2 -b}{b}} (N + 1)^{\frac{2}{b}} \right)^{-1} \right)\\
	& \leq \exp \left( - \left( \frac{B_0^2\theta^{\frac{2(m-1)}{b}}}{4} \log \left( \frac{\log \theta^{m-1}}{\beta}\right) \right)\left(\frac{B_0^2}{3} \left(\frac{N}{\Log(N)}\right)^{\frac{1}{b}} \sqrt{2\theta^{\frac{2m}{b}} \log \left( \frac{\log \theta^{m}}{\beta}\right)} +  u\left(\frac{B_0^b}{\Log(1)}\right)^{\frac{2 -b}{b}} (2N)^{\frac{2}{b}} \right)^{-1} \right)\\
	& \leq \exp \left( - \left( \frac{B_0^2\theta^{\frac{2(m-1)}{b}}}{4} \log \left( \frac{\log \theta^{m-1}}{\beta}\right) \right)\left(\frac{\sqrt{2}B_0^2}{3} \frac{N^{\frac{2}{b}}}{10} +  u\left(\frac{B_0^b}{\Log (1)}\right)^{\frac{2 -b}{b}} (2N)^{\frac{2}{b}} \right)^{-1} \right)\\
	& \leq \exp \left( - \left( \frac{B_0^2\theta^{\frac{2(m-1)}{b}}}{4} \log \left( \frac{\log \theta^{m-1}}{\beta}\right) \right)\left(\frac{B_0^2}{10} N^{\frac{2}{b}} \right)^{-1} \right)
\end{align*}
where the third step uses the fact that $\displaystyle \frac{\sqrt{\log (\log x)}}{(\Log(x))^{1/b}} \leq \frac{1}{10}$ for all $x$ and the last step follows from the value of $B_0$. Thus,
\begin{align*}
\Pr(\text{err}) & \leq  \sum_{m = m_0}^{\infty} \exp \left( - \frac{10\theta^{\frac{2(m-1)}{b}} }{4 \theta^{2m/b}}  \log \left( \frac{\log \theta^{m-1}}{\beta}\right) \right) \\
& \leq  \sum_{m = m_0}^{\infty} \exp \left( -   \frac{2.5}{\theta^{\frac{2}{b}}} \log \left( \frac{(m-1)\log (\theta)}{\beta} \right) \right) \\
& \leq  \sum_{m = m_0}^{\infty} \left( \frac{2\beta}{b(m-1) \log(1.25)} \right)^{2} \\
& \leq  \left(\frac{\sqrt{\bp}}{6\log(1.25)}\right)^2 \sum_{m = 1}^{\infty} \frac{1}{m^2} \\
& \leq  \bp \left(\frac{1}{6\log(1.25)}\right)^2 \frac{\pi^2}{6} \\
& \leq  \bp 
\end{align*}

We now prove the upper bound on the number of samples required by the local test for the heavy tailed noise. The proof closely follows that of the finite variance case. Define $\displaystyle s_0 =  (2B_0^2)^{\frac{b}{2(b-1)}} {\left(\frac{\mu}{2}\right)^{\frac{b}{1-b}}} \log^{\frac{b}{2(b-1)}} \left( \frac{3b}{2(b - 1)\beta} \log \left( \sqrt{\frac{{12} B_0^2 b }{\mu^2 (b -1) \beta }}\right)\right)$ and $s_1 = \lceil s_0 \rceil$ with $\displaystyle \beta = \frac{b\sqrt{\bp}}{12}$ as before. Consider the first term of the threshold given as
\begin{align*}
	\hat{\tau}(s_0, \beta) & = \sqrt{\frac{B_0^2}{2}s_1^{\frac{2-2b}{b}} \log \left( \frac{\log s_1}{\beta}\right)} \\
	& \leq \sqrt{\frac{B_0^2}{2}s_0^{\frac{2-2b}{b}} \log \left( \frac{\log s_0}{\beta}\right)} \\
	& \leq \frac{B_0}{\sqrt{2}}s_0^{\frac{1-b}{b}} \sqrt{\log \left( \frac{\log s_0}{\beta}\right)}  \\
	& \leq \frac{B_0}{\sqrt{2}}s_0^{\frac{1-b}{b}} \sqrt{\log \left( \frac{1}{\beta} \log \left( (2B_0^2)^{\frac{b}{2(b-1)}} {\left(\frac{\mu}{2}\right)^{\frac{b}{1-b}}} \log^{\frac{b}{2(b-1)}} \left( \frac{3b}{2(b - 1)\beta} \log \left( \sqrt{\frac{{12}B_0^2 b}{\mu^2 (b -1) \beta }} \right)\right) \right)\right)} \\
	& \leq \frac{B_0}{\sqrt{2}}s_0^{\frac{1-b}{b}} \sqrt{\log \left( \frac{b}{2(b-1)\beta} \log \left( \frac{{8}B_0^2}{\mu^2} \right) + \frac{b}{2(b-1)\beta} \log \left( \log \left( \frac{3b}{2(b - 1)\beta} \log \left( \sqrt{\frac{{12}B_0^2b}{\mu^2( b- 1) \beta}} \right)\right) \right)\right)}  \\
	& \leq \frac{B_0}{\sqrt{2}}s_0^{\frac{1-b}{b}} \sqrt{\log \left( \frac{b}{2(b-1)\beta} \log \left( \frac{{8}B_0^2}{\mu^2} \right) + \frac{b}{2(b-1)\beta} \log \left(  \frac{3b}{2(b - 1)\beta}  \sqrt{\frac{{12}B_0^2b}{\mu^2( b- 1) \beta}}  \right)\right)}  \\
	& \leq \frac{B_0}{\sqrt{2}}\frac{\mu}{{2}\sqrt{2}B_0}  \left( \log \left( \frac{3{b}}{2({b} - 1)\beta} \log \left( \sqrt{\frac{{12}B_0^2b}{\mu^2( b- 1) \beta}} \right)\right) \right)^{-\frac{1}{2}} \sqrt{\log \left( \frac{b}{2(b-1)\beta}  \log \left( \ \sqrt{\frac{{1728}b^3 B_0^6}{\mu^6(b - 1)^3\beta^3}} \right)  \right)}  \\
	& \leq \frac{\mu}{{4}} \label{mu2omHT}
\end{align*}

Now, we have to similarly consider the second term of the threshold, given as 
\begin{align*}
    \frac{1}{s_1}\sum_{t =1}^{s_1} \frac{u}{B_t^{b-1}} & \leq \frac{u}{s_1 B_0^{b-1}} \sum_{t =1}^{s_1} \left(\frac{\Log(t)}{t} \right)^{\frac{b-1}{b}} \\
    & \leq \frac{u}{s_1 B_0^{b-1}} (\Log(s_1))^{\frac{b-1}{b}} \sum_{t =1}^{s_1} t^{\frac{1-b}{b}} \\
    & \leq \frac{bu}{ B_0^{b-1}} (\Log(s_1))^{\frac{b-1}{b}} s_1^{\frac{1-b}{b}} \\
    & \leq \frac{bu}{ B_0^{b-1}} \left(\frac{\Log(s_1)}{s_1} \right)^{\frac{b-1}{b}} \\
    & \leq \frac{bu}{ B_0^{b-1}} \left(\frac{\Log(s_0)}{s_0} \right)^{\frac{b-1}{b}} \\
    & \leq \frac{bu}{ B_0^{b-1}} s_0^{\frac{1-b}{b}} \sqrt{\Log(s_0) } \\
    & \leq \frac{bu10^{b/2}}{ B_0^{b-1}} s_0^{\frac{1-b}{b}} \sqrt{\log \left( \frac{\log s_0}{\beta}\right)}\\
    & \leq \frac{B_0}{2\sqrt{2}} s_0^{\frac{1-b}{b}} \sqrt{\log \left( \frac{\log s_0}{\beta}\right)}\\
    & \leq \frac{\mu}{8}
\end{align*}
Here the sixth step uses the fact that $\Log(s) \geq 1$ for all $s \geq 0$ and the last but one step follows from bound on the value of $B_0$ and by ensuring $s_0 \geq 8$. The last step follows from the analysis in the previous part.

Now, if $\tau$ denotes the random number of samples taken in the local test before termination, then for $n\ge s_1$,
\begin{eqnarray}\nn
&&\hspace{-3em}\Pr[\tau> n]\\\nn
&\le&\Pr\bigg[\forall s\le n:\overline{Z}_{s} + \sqrt{\frac{B_0^2}{2}s^{\frac{2-2b}{b}} \log \left( \frac{12 \log s}{b \sqrt{ \bp}}\right)} {+} \frac{1}{s} \sum_{t = 1}^{s} \frac{u}{B_t^{b-1}} >0,~\text{and}~ \\ \nn
& & \ \ \ \ \overline{Z}_{s}- \sqrt{\frac{B_0^2}{2}s^{\frac{2-2b}{b}} \log \left( \frac{12 \log s}{b  \sqrt{ \bp}}\right)} {-} \frac{1}{s} \sum_{t = 1}^{s} \frac{u}{B_t^{b-1}} < 0\bigg ]\\\nn 
&\le&\Pr\bigg[\forall s \le n:\overline{Z}_{s} - \mu - \sqrt{\frac{B_0^2}{2}s^{\frac{2-2b}{b}} \log \left( \frac{12 \log s}{b\sqrt{ \bp}}\right)} {-} \frac{1}{s} \sum_{t = 1}^{s} \frac{u}{B_t^{b-1}} < - \mu \bigg ]\\ \nn
&\le& \Pr\bigg[\overline{Z}_{n} - \mu - \sqrt{\frac{B_0^2}{2}n^{\frac{2-2b}{b}} \log \left( \frac{12 \log n}{b\sqrt{ \bp}}\right)} {-} \frac{1}{n} \sum_{t = 1}^{n} \frac{u}{B_t^{b-1}} < - \mu \bigg ]\\ \nn
&\le& \Pr\bigg[\frac{\psi_n}{n} - \frac{1}{n}\sum_{t = 1}^n \E[X \1_{\{|X| > B_t\}}]  <  \sqrt{\frac{B_0^2}{2}n^{\frac{2-2b}{b}} \log \left( \frac{12 \log n}{b\sqrt{ \bp}}\right)} + \frac{1}{n} \sum_{t = 1}^{n} \frac{u}{B_t^{b-1}} - \mu \bigg ]\\ \nn
&\le& \Pr\bigg[\frac{\psi_n}{n}  <  \sqrt{\frac{B_0^2}{2}n^{\frac{2-2b}{b}} \log \left( \frac{12 \log n}{b\sqrt{ \bp}}\right)} + \frac{2}{n} \sum_{t = 1}^{n} \frac{u}{B_t^{b-1}} - \mu \bigg ]\\ \nn
&\le& \Pr\bigg[\frac{\psi_n}{n}  <  \frac{\mu}{4} + \frac{2\mu}{8} - \mu \bigg ]\\ \nn
&\le&\Pr\bigg[\psi_n  < -\frac{n \mu}{2}\bigg]\\ \nn 
&\le&\exp \left( -\frac{1}{2} \left(\frac{n \mu}{2} \right)^2 \left( \frac{2B_n}{3} \frac{n \mu}{2} + \sum_{t = 1}^n u B_t^{2-b} \right)^{-1}\right)\\ \nn 
&\le&\exp \left( -\frac{1}{2} \left(\frac{n \mu}{2} \right)^2 \left(  \frac{n \mu}{3} u B_0 \left( \frac{n}{\Log(n)}\right)^{\frac{1}{b}} + u B_0^{2 -b}\sum_{t = 1}^n  \left( \frac{t}{\Log(t)}\right)^{\frac{2-b}{b}} \right)^{-1}\right)\\ \nn 
&\le&\exp \left( -\frac{1}{2} \left(\frac{n \mu}{2} \right)^2 \left(  \frac{n \mu}{3} u B_0 \left( \frac{n}{\Log(n)}\right)^{\frac{1}{b}} + u n B_0^{2 -b}  \left( \frac{n}{\Log(n)}\right)^{\frac{2-b}{b}} \right)^{-1}\right)\\ \nn 
&\le&\exp \left( -\frac{1}{2} \left(\frac{n \mu}{2} \right)^2 \left(  nu B_0 \left( \frac{n}{\Log(n)}\right)^{\frac{1}{b}} \left\{\frac{ \mu}{3}  +  B_0^{1 -b}  \left( \frac{n}{\Log(n)}\right)^{\frac{1-b}{b}} \right\} \right)^{-1}\right)\\ \nn 
&\le&\exp \left( -\frac{1}{2} \left(\frac{n \mu}{2} \right)^2 \left(  nu B_0 \left( \frac{n}{\Log(n)}\right)^{\frac{1}{b}} \left\{\frac{ \mu}{3}  +  \frac{1}{B_0^{b-1}}  \left( \frac{\Log(s_0)}{s_0}\right)^{\frac{b-1}{b}} \right\} \right)^{-1}\right)\\ \nn 
&\le&exp \left( -\frac{1}{2} \left(\frac{n \mu}{2} \right)^2 \left(  n^{1 + \frac{1}{b}} u B_0  \left\{\frac{ \mu}{3}  +  \frac{\mu}{8 b u}\right\} \right)^{-1}\right)\\ \nn 
&\le&\exp \left(- \frac{\mu}{8 \hat{u}B_0} n^{1 - \frac{1}{b}} \right) 
\end{eqnarray}
where $\displaystyle \hat{u} = \frac{u}{3} + \frac{1}{8}$. We can write $\mathbb{E}[\tau]$ in terms of the sum of the tail probabilities $\Pr[\tau\ge n]$ as
\begin{align*}
    \mathbb{E}[\tau] & = \sum_{n= 0}^\infty\Pr[\tau\ge n]  \\
    & \leq s_1 + \sum_{n= s_1+1}^\infty\Pr[\tau\ge n]  \\ 
    & \leq s_1 + \sum_{n= s_1+1}^\infty \exp \left(- \frac{\mu}{8 \hat{u}B_0} n^{1 - \frac{1}{b}} \right)\\
    & \leq s_1 + \int_{s_1}^{\infty}\exp \left(- \frac{\mu}{8 \hat{u}B_0} x^{1 - \frac{1}{b}} \right) dx \\
    & \leq s_1 + \frac{b}{b-1}\int_{s_1^{(1 - 1/b)}}^{\infty} \exp \left(- \frac{\mu}{8 \hat{u} B_0} y \right) y^{\frac{1}{b-1}}  dy \\
    & \leq s_1 + \frac{b}{b-1} \left(\frac{8\hat{u}B_0}{\mu} \right)^{\frac{b}{b -1}} \int_{0}^{\infty} \exp \left(- z \right) z^{\frac{1}{b-1}}  dz \\
    & \leq s_1 + \frac{b}{b-1} \left(\frac{8\hat{u}B_0}{\mu} \right)^{\frac{b}{b -1}} \Gamma \left( \frac{b}{b -1} \right)\\
    & \leq \left(\Gamma \left(\frac{2b - 1}{b -1}\right) \left(\frac{\hat{u}}{B_0}\right)^{\frac{b}{b-1}} + 1\right) s_1
\end{align*}
where $\displaystyle \Gamma(z) = \int_0^{\infty} x^{z-1} e^{-z} dz$ is the Gamma function.
Plugging in the values of $s_1$ and $\beta$, we arrive at $\displaystyle \E[ \tau(x) ] \leq {\left(\Gamma \left(\frac{2b - 1}{b -1}\right) \left(\frac{{1}}{B_0}\left( \frac{u}{3} + \frac{1}{8} \right)\right)^{\frac{b}{b-1}} + 1\right) } \left(\left( \frac{{8}B_0^2}{|g(x)|^2} \log \left( \frac{18}{(b - 1)\sqrt{\bp}} \log \left( {\frac{{144}B_0^2 }{|g(x)|^2 (b - 1)\sqrt{\bp}  }} \right) \right) \right)^{\frac{b}{2(b-1)}}      + {8}\right) $.

\end{proof}

\section*{Appendix D}

\begin{proof}[Proof of Theorem 1]
Let $\tau_m$ denote the number of samples taken in the $m$th time that the local sequential test is carried out; $\tau_1$ samples are taken at point $x_{(1)}$, $\tau_2$ samples are taken at point $x_{(2)}$ and so on. 
Let $t_n$ denote the time at the end of the $n$th step of the random walk: $t_n=\sum_{m=1}^{3n}\tau_{m}$. Notice that both $\tau_m$ and $t_n$ are random variables were the randomness comes from the randomness in the samples of $g$. Define
$n_0=\max\{{n\in \mathbb{N}}:n\le \frac{3}{(2p-1)^2}\log t_n, t_n\le T\}$ and $t_0=t_{n_0}$. The definition of $n_0$ and $t_0$ indicates that at $t>t_0$, the random walk has taken more than $\frac{3}{(2p-1)^2}\log t$ steps. We analyze the regret incurred up to time $t_0$, and after that, separately.
\begin{eqnarray}\nn
\mathcal{R}_{\mbox{\footnotesize RWGD}}(T)
&=&\E \left[\sum_{t=1}^T \left(f(x_{\pi(t)}-f(x^*))\right) \right] \\\label{RegExp}
&=&\underbrace{\E \left[\sum_{t=1}^{t_0}\left(f(x_{\pi(t)}-f(x^*))\right)\right]}_{R_1}
+ \underbrace{\E\left[\sum_{t=t_0+1}^{T}\left(f(x_{\pi(t)}-f(x^*))\right) \right]}_{R_2}.
\end{eqnarray}
Next, we establish an upper bound on each term of the regret.

\emph{Upper bound on the first term $R_1$}.
From Lemma~\ref{SCofP}, we have that $\tau_m$ satisfies
\begin{eqnarray}
\mathbb{E}[\tau_m]\le
\frac{40 \sigma^2}{ g^2(x_m)} \log \left( \frac{12}{ \sqrt{ \bp}} \log \left( \frac{240 \sigma^2 }{\sqrt{ \bp}  g^2(x_m)} \right)\right) + 2.
\end{eqnarray}
Based on this upper bound on $\E[\tau_m]$, 

I. when $f$ is convex, we have
\begin{eqnarray}\nn
&&\hspace{-3em}R_1 \le \sum_{m=1}^{3n_0} \E[\tau_m] g(x_m)\Delta_{x_{(m)}}.~~~~~
\end{eqnarray}

II. when $f$ is strongly convex 
\begin{eqnarray}\nn
&&\hspace{-3em}R_1 \le \sum_{m=1}^{3n_0}\E[\tau_m]  \frac{g^2(x_m)}{2\alpha}.~~~~~
\end{eqnarray}

Noticing the constraint that $\sum_{m=1}^{3n_0}\E[\tau_m]=\E[t_0]\le T$ and using $\Delta_x\le 1$, the following constrained optimization problems gives us an upper bound on $R_1$.
{{\begin{eqnarray}\label{max1}
&&\hspace{-2em}\max_{\{g(x_{(m)}),\tau_m\}_{m=1}^{3n_0}} \sum_{m=1}^{3n_0} \bigg(
\E[\tau_m]g(x_{(m)})
\bigg), ~\text{when $f$ is convex}\\\label{max2}
&&\hspace{-2em}\max_{\{g(x_{(m)}),\tau_m\}_{m=1}^{3n_0}} \sum_{m=1}^{3n_0} \bigg(
\E[\tau_m]\frac{g^2(x_{(m)})}{2\alpha}
\bigg), ~\text{when $f$ is strongly convex}\\\nn
&&\hspace{-2em}\textit{subject to:} \sum_{m=1}^{3n_0}\E[\tau_m]\le T,~~~\text{and}~~\\\nn
&&\hspace{-2em}\mathbb{E}[\tau_m]\le
\frac{40 \sigma^2}{ g^2(x_m)} \log \left( \frac{12}{ \sqrt{ \bp}} \log \left( \frac{240 \sigma^2 }{\sqrt{ \bp}  g(x_m)^2} \right)\right) + 2~\text{for all}~m. 
\end{eqnarray}}}


Using standard Lagrangian method it can be shown that for the optimum values we have $\E[\tau_m]=\frac{T}{3n_0}$ (for all $m$). Using the upper bound on $\E[\tau_m]$ we have the following upper bound on $g^2(x_m)$



\begin{eqnarray}\label{gUB}
g^2(x_m) \le \frac{40 \sigma^2}{ \E[\tau_m]} \log \left( \frac{12}{ \sqrt{ \bp}} \log \left( \frac{6\E[\tau_m] }{\sqrt{ \bp}  } \right)\right) + \frac{2g^2_{\max}}{\E[\tau_m]}.
\end{eqnarray}

Substituting $g(x_m)$ from~\eqref{gUB} and $\E[\tau_m]=\frac{T}{3n_0}$ in the optimization problem
(together with $n_0 \le  \frac{3}{(2p-1)^2}\log T$) results in

I. when $f$ is convex

\begin{eqnarray}\label{case1}
R_1&\le& \sqrt{ \frac{360\sigma^2T\log T}{(2p-1)^2}
\log\left(\frac{12}{\sqrt{\bp}}\log\frac{2(2p-1)^2T}{\sqrt{\bp}3\log T}\right)}
+ \sqrt{\frac{18g^2_{\max}T\log T}{(2p-1)^2}}
.
\end{eqnarray}

I. when $f$ is strongly convex

\begin{eqnarray}\label{case2}
&&\hspace{-2.5em}{R}_{1}(T)\le
\frac{360\sigma^2\log T}{2\alpha(2p-1)^2}\log\left(\frac{12}{\sqrt{\bp}}\log\frac{2(2p-1)^2T}{3\log T\sqrt{\bp}}\right)+ \frac{18g^2_{\max}\log T}{2\alpha(2p-1)^2}
\end{eqnarray}

For the case where $f$ is non-differentiable at $x^*$ with a $\delta>0$ lower bound on gradient, from the upper bound on $\E[\tau]$, we have
\begin{eqnarray}\label{case3}
&&\hspace{-2em}{R}_{1}(T)\le
\frac{9g_{\footnotesize \mbox{max}}\log T}{(2p-1)^2}\bigg(\frac{40\sigma^2}{\delta^2}\log\left(\frac{12}{\sqrt{\bp}}\log\frac{240\sigma^2}{\sqrt{\bp}\delta^2}\right)+2\bigg)+ g_{\footnotesize \mbox{max}}(\log T+4). 
\end{eqnarray}

\emph{Upper bound on the second term $R_2$}.
At time $t: t_0<t\le T$, by definition of $t_0$ we know that the random walk has taken more than $\frac{3}{(2p-1)^2}\log t$ steps. 
From Lemma~\ref{Lemma2}, we have with probability at least$~1-\exp(-\frac{\frac{3}{(2p-1)^2}\log (t)(2p-1)^2}{2})=1-(\frac{1}{t})^{\frac{3}{2}}$, 
\begin{eqnarray}\nn
\Delta_{x_{\pi(t)}}&\le&  2^{-\frac{\frac{3}{(2p-1)^2}\log (t)(2p-1)}{2}}
=\left(\frac{1}{t}\right)^{\frac{3\log2}{2(2p-1)}}
\le\frac{1}{t},
\end{eqnarray}
where the last inequality is obtained by $p<1$ and $\frac{3\log 2}{2}<1$. 

The second term $R_2$ in the regret is upper bounded as follows. 
\begin{eqnarray}\nn
R_2 &=&\E\left[\sum_{t=t_0+1}^{T}\left(f(x_{\pi(t)}-f(x^*))\right) \right]\\\nn
&=&\E\left[
\sum_{t=t_0+1}^{T}\left(f(x_{\pi(t)}-f(x^*))\right)\I[\Delta_{x_{\pi(t)}}\le \frac{1}{t}]
\right]
+\E\left[
\sum_{t=t_0+1}^{T}\left(f(x_{\pi(t)}-f(x^*))\right)\I[\Delta_{x_{\pi(t)}}> \frac{1}{t}]
\right]\\\nn
&\le& \sum_{t=1}^{T} G\left(\frac{1}{t} + \left(\frac{1}{t}\right)^{\frac{3}{2}}\right)\\\label{R2UB}
&\le& g_{\max}(\log T+4).
\end{eqnarray}
In the above inequalities $\I$ is the indicator function. We used the convexity of $\E[f]$ and $\Delta_x\le1$ to arrive at~\eqref{R2UB}.

The upper bound on~\eqref{case1},~\eqref{case2} and~\eqref{case3} together with~\eqref{R2UB} results in Theorem~\ref{The:SubG}. 
\end{proof}

\section*{Appendix E}

\begin{proof}[Proof of Theorem~\ref{The:HT}]
The proof of this theorem follows the exact arguments as the proof of Theorem~\ref{The:SubG}, except for the upper bound on $\E[\tau]$. In particular plugging in the new upper bounds, the following maximization problem (in place of~\eqref{max1} and~\eqref{max1}) 
gives us an upper bound on $R_1$.
{{\begin{eqnarray}\label{max3}
&&\hspace{-2em}\max_{\{g(x_{(m)}),\tau_m\}_{m=1}^{3n_0}} \sum_{m=1}^{3n_0} \bigg(
\E[\tau_m]g(x_{(m)})
\bigg), ~\text{when $f$ is convex}\\\label{max4}
&&\hspace{-2em}\max_{\{g(x_{(m)}),\tau_m\}_{m=1}^{3n_0}} \sum_{m=1}^{3n_0} \bigg(
\E[\tau_m]\frac{g^2(x_{(m)})}{2\alpha}
\bigg), ~\text{when $f$ is strongly convex}\\\nn
&&\hspace{-2em}\textit{subject to:} \sum_{m=1}^{3n_0}\E[\tau_m]\le T,~~~\text{and}~~\\\nn
&&\hspace{-2em}\E[\tau(x)]  \leq  {\left(\Gamma \left(\frac{2b - 1}{b -1}\right) \left(\frac{{1}}{B_0}\left( \frac{u}{3} + \frac{1}{8} \right)\right)^{\frac{b}{b-1}} + 1\right) } \left(\left( \frac{{8}B_0^2}{|g(x)|^2} \log \left( \frac{18}{(b - 1)\sqrt{\bp}} \log \left( {\frac{{144}B_0^2 }{|g(x)|^2 (b - 1)\sqrt{\bp}  }} \right) \right) \right)^{\frac{b}{2(b-1)}}      + {8}\right)
\end{eqnarray}}}
Similar to the proof of Theorem~\ref{The:HT}, choosing $\E[\tau_m] = \frac{T}{3n_0}$ for all $m$ and characterizing the corresponding upper bounds on $g(x_m)$ gives us Theorem~\ref{The:HT}.

\end{proof}

\end{document}